\documentclass{article}

\usepackage[final]{neurips_2022}

\usepackage[utf8]{inputenc} 
\usepackage[T1]{fontenc}    
\usepackage[hidelinks]{hyperref}       
\usepackage{url}            
\usepackage{booktabs}       
\usepackage{amsfonts}       
\usepackage{nicefrac}       
\usepackage{microtype}      
\usepackage{xcolor}         

\usepackage{proj_macros}

\usepackage{algorithm}
\usepackage{algorithmic}

\usepackage{graphicx}
\usepackage{subfigure}
\usepackage{adjustbox}

\usepackage{svg}
\usepackage{natbib}
\usepackage{multirow}
\usepackage{wrapfig}

\newcommand{\imagesfolder}{images}

\renewcommand{\a}{\mathbf{a}}

\newcommand{\etabf}{{\boldsymbol\eta}}
\newcommand{\mgen}{m_\text{gen}}

\title{A Combinatorial Perspective on the Optimization of Shallow ReLU Networks}

\author{%
  Michael Matena \\
  Department of Computer Science\\
  University of North Carolina at Chapel Hill\\
  Chapel Hill, NC 27599 \\
  \texttt{mmatena@cs.unc.edu} \\
   \And
  Colin Raffel \\
  Department of Computer Science\\
  University of North Carolina at Chapel Hill\\
  Chapel Hill, NC 27599 \\
  \texttt{craffel@cs.unc.edu} \\
}

\begin{document}

\maketitle

\begin{abstract}
The NP-hard problem of optimizing a shallow ReLU network can be characterized as a combinatorial search over each training example's  activation pattern followed by a constrained convex problem given a fixed set of activation patterns.
We explore the implications of this combinatorial aspect of ReLU optimization in this work.
We show that it can be naturally modeled via a geometric and combinatoric object known as a zonotope with its vertex set isomorphic to the set of feasible activation patterns.
This assists in analysis and provides a foundation for further research.
We demonstrate its usefulness when we explore the sensitivity of the optimal loss to perturbations of the training data.
Later we discuss methods of zonotope vertex selection and its relevance to optimization.
Overparameterization assists in training by making a randomly chosen vertex more likely to contain a good solution.
We then introduce a novel polynomial-time vertex selection procedure that provably picks a vertex containing the global optimum using only double the minimum number of parameters required to fit the data.
We further introduce a local greedy search heuristic over zonotope vertices and demonstrate that it outperforms gradient descent on underparameterized problems.
\end{abstract}

\section{Introduction}
Neural networks have become commonplace in a variety of applications.
They are typically trained to minimize a loss on a given dataset of labeled examples using a variant of stochastic gradient descent.
However, our theoretical knowledge of neural networks and their training lags behind their practical developments.
 
Single-layer ReLU networks are an appealing subject for theoretical study.
The universal approximation theorem guarantees their expressive power while their relative simplicity makes analysis tractable \citep{hornik1991approximation}.
We restrict ourselves in this paper to studying empirical risk minimization (ERM) as was done in previous works \citep{du2018gradient, oymak2020toward}, which is justified since the train set performance tends to upper bound the test set performance.
Furthermore, modern neural networks achieve zero training loss but nevertheless generalize well \citep{kaplan2020scaling,nakkiran2021deep}.
Minimizing the training loss of a shallow ReLU network is a nonconvex optimization problem.
Finding its global minima is difficult and can in fact be shown to be NP-hard in general \citep{goel2020tight}.
\Citet{arora2016understanding} provide an explicit algorithm for finding the global minima by solving a set of convex optimization problems; however, the size of this set is exponential in both the input dimension $d$ and the number of hidden units $m$.



In this paper, we explore the combinatorial structure implicit in the global optimization algorithm of \citet{arora2016understanding}.
We start by using tools from polyhedral geometry to characterize the set of convex optimization problems and describe the relationships between the subproblems.
Notably, we are able to create a special type of polytope called a zonotope \citep{mcmullen1971zonotopes} whose vertices have a one-to-one correspondence with the convex subproblems and whose faces represent information about their relationships.
We then explore the combinatorial optimization problem implicit in shallow ReLU network empirical risk minimization using the zonotope formalism to help interpret our findings and assist in some proofs.

Since the computational complexity of optimization problems shapes our approach to solving them, we examine the reductions of NP-hard problems introduced in \citet{goel2020tight}.
The datasets produced have examples that are not in general position (i.e.\ they have nontrivial affine dependencies), which differs from most real-world datasets.
We prove that the global optimum of the loss of a shallow ReLU network over such a dataset can have a discontinuous jump for arbitrarily small perturbations of the data, which has a very natural interpretation in our zonotope formalism.
This means that a proof of the NP-hardness of ReLU optimization given training examples in general position does not follow from the results \citet{goel2020tight} via a simple continuity argument.
We therefore present a modification of their proof that uses a dataset in general position.



In contrast to the NP-hardness of general ReLU optimization, sufficient overparameterization allows gradient descent to provably converge to a global optimum in polynomial time, as demonstrated by \citep{du2018gradient, oymak2020toward}.
We interpret the proof methods generally used in these works as asserting that sufficient overparameterization allows gradient descent to bypass much of the combinatorial search over zonotope vertices by having a randomly chosen vertex be close to one with a good solution with high probability.
We then introduce a novel algorithm that finds a good zonotope vertex in polynomial time 
requiring only about twice the minimum number of hidden units required to fit the dataset.


Finally, we explore how gradient descent interacts with this combinatorial structure.
We provide empirical evidence that it can perform some aspects of combinatorial search but present an informal argument that it is suboptimal.
We reinforce this claim by showing that a greedy local search heuristic over the vertices of the zonotope outperforms gradient descent on some toy synthetic problems and simplifications of real-world tasks.
In contrast to the NP-hard worst case, these results suggest that the combinatorial searches encountered in practice might be relatively tractable.


We summarize our contributions as follows.
\begin{itemize}[noitemsep,topsep=0pt,parsep=2pt,partopsep=0pt,left=1em..2em]
    \item We are the first to provide an in-depth exposition of the combinatorial structure arising from the set of feasible activation patterns that is implicit in shallow ReLU network optimization. In particular, we show that this structure can be characterized exactly as a Cartesian power of the zonotope generated by the set of training examples.
    \item We use this formalism to prove necessary conditions for the global optimum of a shallow ReLU network to be discontinuous with respect to the training dataset. We show that this implies that previous NP-hardness proofs of ReLU optimization do not automatically apply to datasets satisfying realistic assumptions, which we rectify by presenting a modification that uses a dataset in general position.
    \item We explore the role that combinatorial considerations play in the relationship between overparameterization and optimization difficulty. In particular, we introduce a novel polynomial-time algorithm fitting a generic dataset using twice the minimum number of parameters needed.
    \item We introduce a novel heuristic algorithm that performs a greedy search along edges of a zonotope and show that it outperforms gradient descent on some toy datasets.
\end{itemize}

We hope that the tools we introduce are generally useful in furthering our understanding of ReLU networks.
Notably, they have deep connections to several well-established areas of mathematics \citep{mcmullen1971zonotopes, richter20176, ziegler2012lectures}, which might allow researchers to quickly make new insights by drawing upon existing results in those fields.




\section{Empirical Risk Minimization for ReLU Networks}



A single ReLU layer consists of an affine transformation followed by a coordinate-wise application of the ReLU nonlinearity $\phi(x) = \max\{x, 0\}$.
We can represent an affine transformation from $\R^d \to \R^m$ by an $m \times (d+1)$ matrix by representing its inputs in homogeneous coordinates, i.e.\ by appending a $(d+1)$-th coordinate to network inputs that is always equal to 1.
Hence a single ReLU layer with parameters $W$ can be written as $f_W(\x) = \phi(W\bar\x)$, where $\bar\x$ denotes $\x$ expressed in homogeneous coordinates.
A single hidden layer ReLU network consists of a single ReLU layer followed by an affine transformation.
We focus on the case of a network with scalar output, so the second layer can be represented by a vector $\v \in \R^{m+1}$.
Although the second layer parameters are trained jointly with the first layer in practice, we assume that they are fixed for our analysis.
This parallels simplifying assumptions made in previous work \citep{du2018gradient}.

Let $\ell: \R \times \R \to \R$ be a convex loss function such as MSE or cross-entropy.
Since the second layer parameters are fixed, we can incorporate them into a modified loss function $\Tilde\ell: \R^m\times\R\to\R$ given by $\Tilde\ell(\z, y) = \ell(\v^T\bar\z, y)$ that operates directly on the first layer's activations $\z$.
This modified loss function is convex since it is the composition of a convex function with an affine function.

Suppose we are given $\mathcal D = \{(\x_i,y_i)\}_{i=1}^N$ as the training dataset with $\x_i\in\R^d$ and $y_i\in\R$ for all $i = 1, \dotsc, N$.
Throughout this paper, we assume that $N > d + 1$.
Sometimes we will represent a dataset by a matrix $X \in \R^{(d+1)\times N}$ with each column corresponding to an example and its labels as the vector $\y \in \R^N$.
We say that $\mathcal D$ is in general position if there exist no nontrivial affine dependencies between the columns of $X$.
The empirical loss $L(W)$, also known as the empirical risk, is defined as the mean per-example loss
\begin{equation}\label{eq:def_emp_loss}
    L(W) = \frac{1}{N} \sum_{i=1}^N \Tilde\ell(f_W(\x_i), y_i).
\end{equation}

The goal of ERM is to find a set of parameters $W$ that minimizes this loss.
\Citet{arora2016understanding} were the first to introduce an algorithm for exact ERM.
We adapt their algorithm for the case of fixed second layer weights in \cref{alg:arora}, which has a running time of $O(N^{md}\poly(N, m, d))$.
The algorithm works by iterating over all feasible activation patterns
\begin{equation}\label{eq:def_aps}
    \mathcal A = \left\{ \mathbb I\{ W\bar X > 0\} \in \{0,1\}^{m \times N} \mid W \in \R^{m \times (d + 1)} \right\}.
\end{equation}
The subset of parameters corresponding to an activation pattern, which we call an activation region, can be expressed via a set of linear inequalities.
Within a single activation region, the map from parameter values to ReLU layer activations over the training dataset is linear.
Hence we can solve a constrained convex optimization problem to get the optimal parameters in each activation region.
Namely for a given $A \in \mathcal A$, we solve for $W \in \R^{m \times (d+1)}$ in the following
\begin{equation}\label{eq:activation_region_constrained_opt}
\begin{array}{ll@{}ll}
\text{minimize}  & \displaystyle\frac{1}{N} \sum_{i=1}^N \Tilde\ell(\a_i \odot (W\bar\x_i), y_i) \\
\text{subject to} & (2\a_i - 1) \odot (W\bar\x_i) \geq 0 \\
\end{array}
\end{equation}
where $\odot$ denotes the Hadamard product and $\a_i$ denotes the $i$-th column of $A$.
The global optimum then becomes the best optimum found over the entire set of activation regions.
Thus single-layer ReLU network ERM can be expressed as a combinatorial search over activation patterns with a convex optimization step per pattern.

\begin{wrapfigure}{L}{0.55\textwidth}
\begin{minipage}{0.55\textwidth}
\begin{algorithm}[H]
    \caption{Exact ERM \citep{arora2016understanding}}\label{alg:arora}
\begin{algorithmic}
   \STATE {\bfseries Input:} data $\mathcal D = \{\x_i,y_i\}_{i=1}^N$, 2nd layer $\v \in \R^{m + 1}$
   \STATE $\mathcal A \subseteq \{0,1\}^{m \times N}$ \COMMENT{feasible activation patterns \eqref{eq:def_aps}}
   \STATE $W^* \in \R^{m \times (d + 1)}$ \COMMENT{random initialization}
   \FOR{$A \in \mathcal A$}
      \STATE $W$ $\gets$ solution of \eqref{eq:activation_region_constrained_opt}
      \IF{$L(\Tilde W) < L(W^*)$}
          \STATE $W^*$ $\gets$ $W$
      \ENDIF
   \ENDFOR
   \STATE {\bfseries return} $W^*$
\end{algorithmic}
\end{algorithm}
\end{minipage}
\end{wrapfigure}

\section{Zonotope Formalism}\label{sec:comb_struct}
While \citet{arora2016understanding} mention that $\mathcal A$ arises from a set of hyperplanes induced by the training examples, they only use this connection to bound its cardinality $|\mathcal A| = O(N^{md})$.
However, hyperplane arrangements are well-studied geometric and combinatoric objects \citep{richter20176, stanley2004introduction}.
As such, we will see that we can use this connection to better characterize the combinatorial aspects of ReLU optimization.

\begin{figure*}[ht]
\begin{center}
\centerline{
\includegraphics[width=0.3\linewidth]{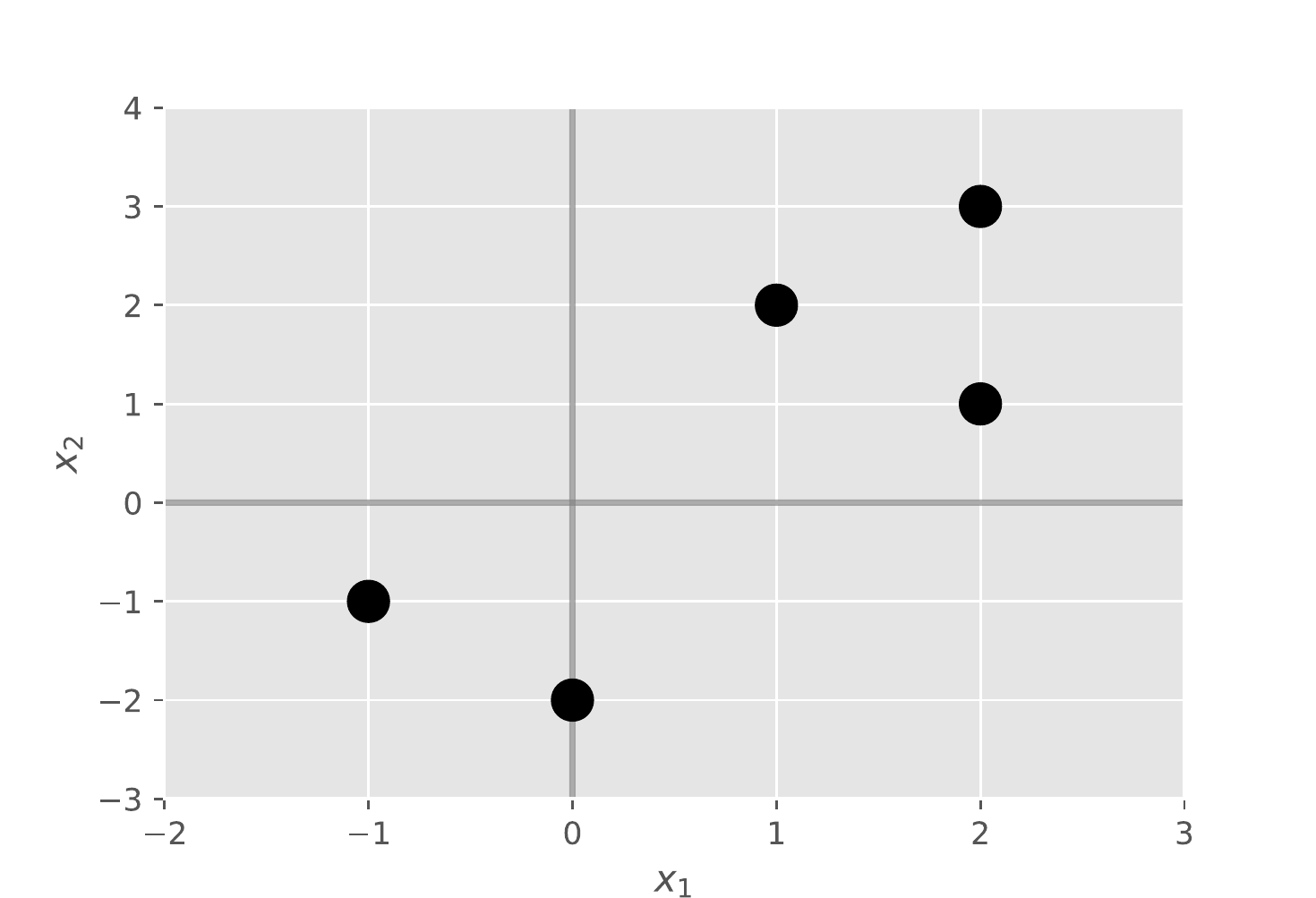}
\includegraphics[width=0.3\linewidth]{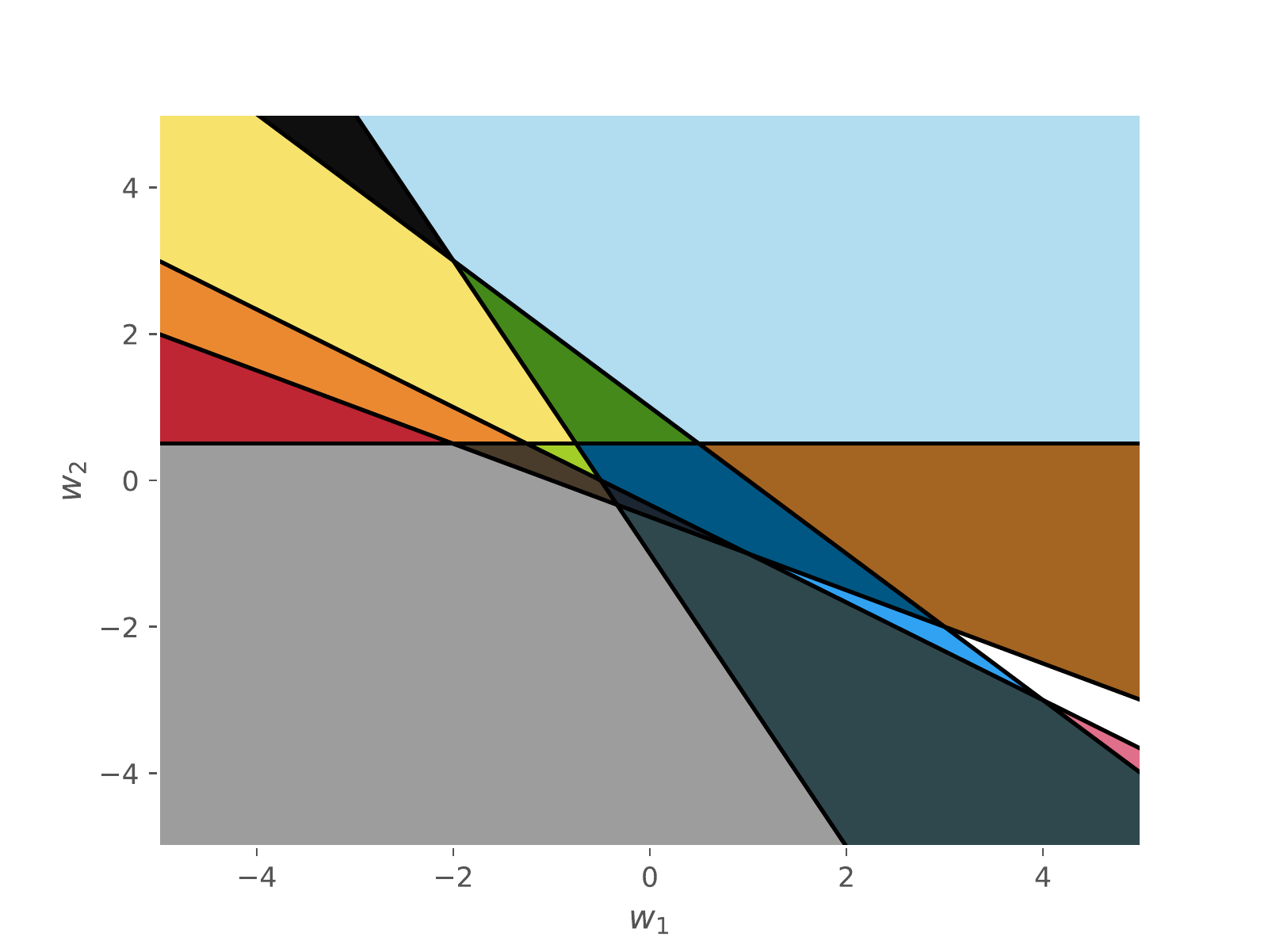}
\includegraphics[width=0.35\linewidth]{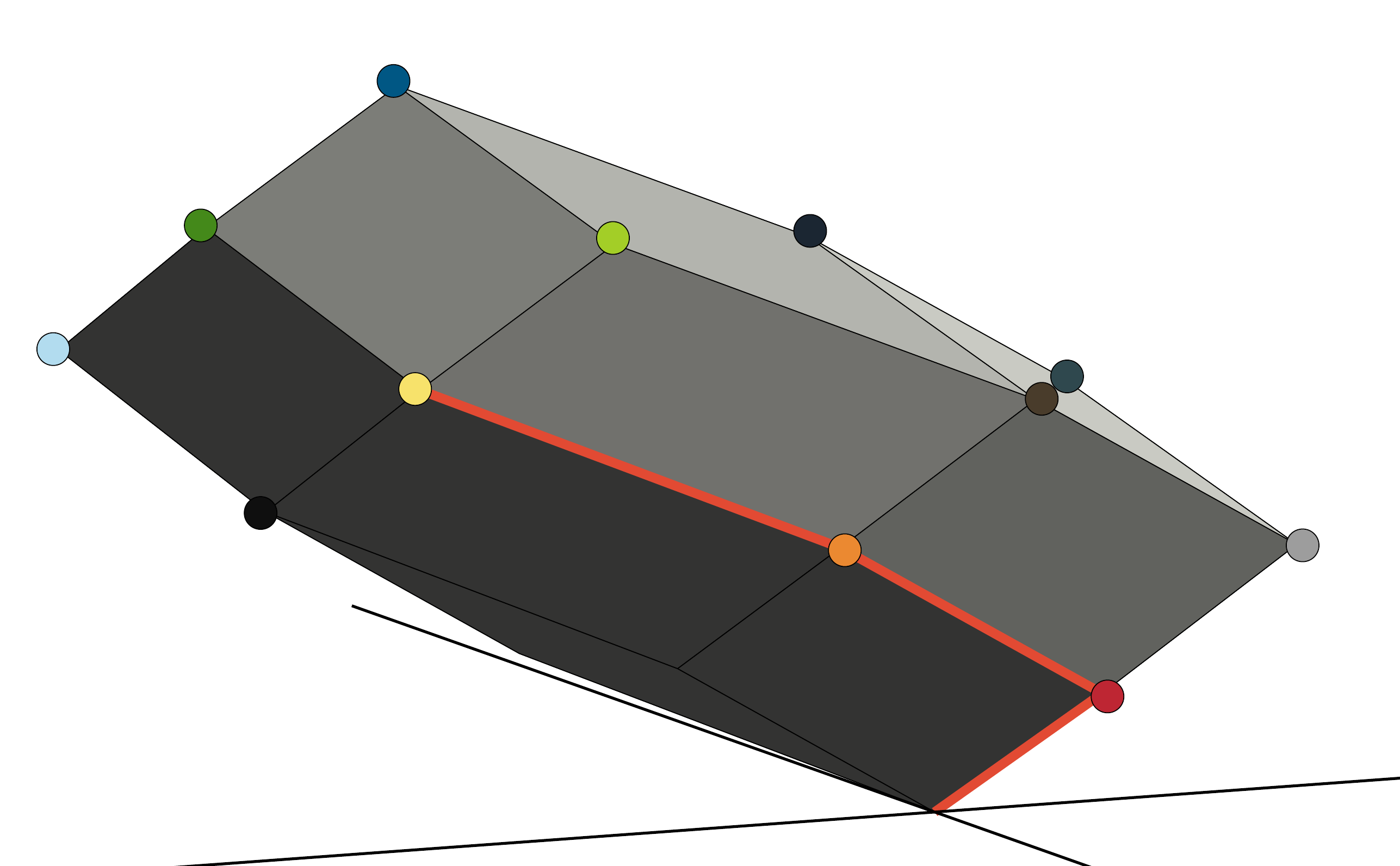}
}
\caption{
Left: A set of 5 training examples in $\R^2$.
Center: A two dimensional slice of parameter space $\R^3$ along the $w_3 = 1$ plane reflecting the polyhedral complex $\mathcal R$. The lines correspond to the set of hyperplanes $H_1^0, \dotsc, H_5^0$. The different shaded chambers correspond to different activation regions. Each chamber can be thought of as the base of cone whose apex is the origin.
Right: The zonotope $\mathcal Z$ for this dataset. The corresponding activation region for a vertex is indicated by the colored circles. Note how the edges and faces of $\mathcal Z$ capture the incidence structure of the activation regions. Each of the red lines is a translation of a (homogenized) training example.
Exactly these 3 training examples are active in the yellow activation region.
}
\label{fig:relu_zono_diagram}
\end{center}
\vskip -0.2in
\end{figure*}



Our mathematical tools for describing the combinatorial structure of shallow ReLU network optimization include oriented hyperplane arrangements, polyhedral sets, polyhedral complexes, and zonotopes.
\Cref{sec:math_background} provides an approachable overview of these topics for unfamiliar readers.

\paragraph*{Single Hidden Unit}

We start by considering a single ReLU unit $f_\w(\x) = \phi(\w^T\bar\x)$ parameterized by the vector $\w \in \R^{d+1}$.
Looking at its behavior as $\w$ ranges over $\R^{d+1}$ on a single training example $\x_i$, we see that there are two linear regimes depending on the sign of $\w^T\bar\x$. They are separated by the hyperplane in parameter space satisfying $\w^T\bar\x = 0$.
We can describe such behavior mathematically as an oriented hyperplane with the sign of $\w^T\bar\x$ providing its orientation.
The collection of oriented hyperplanes associated to each training example $\{\bar\x_i\}_{i=1}^N$ is known as an oriented hyperplane arrangement \citep{richter20176}. 

The structure imposed on parameter space $\R^{d+1}$ by this oriented hyperplane arrangement can be described as a polyhedral complex \citep{ziegler2012lectures}, which is a collection of polyhedral sets and their faces that fit together in a ``nice'' way.
The polyhedral complex $\mathcal R$ induced by the training set will contain codimension 0 sets called chambers.
These correspond exactly to activation regions.
Activation patterns have a one-to-one correspondence with the tuple of hyperplane orientations associated to each chamber.
The center panel of \cref{fig:relu_zono_diagram} provides an illustration of $\mathcal R$ for an example dataset.

The dual zonotope of a polyhedral complex is a single polytope providing an alternate representation of its combinatorial structure \citep{ziegler2012lectures}.
Each dimension $k$ member of the polyhedral complex has a corresponding codimension $k$ face in the dual zonotope.
Incidence relations between the members of the complex are preserved in the dual zonotope.
Generally, a zonotope can be described as the image of an $N$-dimensional hypercube under a linear map whose columns are known as its generators \citep{mcmullen1971zonotopes}.
Each vertex of the zonotope thus is a weighted sum of its generators with coefficients belonging to $\{0,1\}$.

The dual zonotope $\mathcal Z$ of our polyhedral complex $\mathcal R$ has the training examples $\{\bar\x_i\}_{i=1}^N$ as its generators.
The vertices of $\mathcal Z$ have a one-to-one correspondence to the activation regions of our network.
When a vertex is expressed as a weighted sum over the generators, the coefficient $\{0,1\}$ of each generator equals its corresponding example's value in the region's activation pattern.
The right panel of \cref{fig:relu_zono_diagram} shows an example zonotope $\mathcal Z$ and its duality with the polyhedral complex $\mathcal R$.



These correspondences allow us to assign additional structure to the set of activation patterns $\mathcal A$ rather than just treating it as an unstructured set.
For example, the 1-skeleton of the zonotope $\mathcal Z$, which is the graph formed by its vertices and edges, provides a means of traversing the set of activation patterns.
Furthermore, we can directly make connections between the training dataset and the activation pattern structure by making use of the fact that $\mathcal Z$ is generated by the training examples.

\paragraph*{Multiple Hidden Units}

In the multiple hidden unit setting, i.e.\ $m>1$, note that parameter space becomes an $m$-fold Cartesian product of single-unit parameter spaces.
Furthermore, we are free to set the activation pattern for each unit independently of the others.
As the combinatorial structure for each hidden unit can be described using the zonotope $\mathcal Z$, the combinatorial structure for a multiple hidden unit network is described by the $m$-fold Cartesian product
$\mathcal Z^m = \prod_{i=1}^m \mathcal Z$.
As noted in \cref{sec:cart_prod_zono}, $\mathcal Z^m$ is also a zonotope.
Each vertex of $\mathcal Z^m$ corresponds to a product of $m$ vertices of $\mathcal Z$.
As in the single unit case, there is a one-to-one correspondence between the vertices of $\mathcal Z^m$ and the set of activation patterns $\mathcal A$.




\section{ReLU Optimization}

\subsection{NP-Hardness}
Given the additional structure we have imposed on the set of activation patterns in \cref{alg:arora}, it is natural to ask whether we can use it to develop a global optimization algorithm that is more efficient than a brute-force search over activation patterns.
Unfortunately, several works \citep{goel2020tight, froese2021computational} have demonstrated that global optimization of a shallow ReLU network is NP-hard.
%
Nevertheless, this does not preclude the existence of an efficient combinatorial optimizer given certain conditions on the input dataset.
Since the zonotope $\mathcal Z^m$ encapsulates the combinatorial structure of the optimization problem, we look to see if properties of $\mathcal Z^m$ can be related to the difficulty of combinatorial optimization.

Nontrivial affine dependencies between training examples influence the combinatorial structure of the $\mathcal Z^m$.
Since the reductions of NP-hard problems to ReLU optimization done in \citet{goel2020tight, froese2021computational} create datasets with such nontrivial dependencies, it is natural to ask whether it is NP-hard to optimize a shallow ReLU network over a training dataset in general position.

\subsubsection{Discontinuity of the Global Optimum}\label{sec:discont_global_opt}
If the global minimum of the loss is always continuous with respect to the input dataset, then the NP-hardness of optimization over arbitrary datasets in general position would follow from continuity since every set of points is arbitrarily close to a set in general position.
However, we can prove that such continuity holds unconditionally for ReLU optimization only in the case where the training dataset is in general position.
We give a sketch of the proof here along with some analysis of the failure cases that can happen when the data are not in general positions.
We provide a full proof in \cref{sec:stability_proof}.
\begin{theorem}\label{thm:stability_gp}
Suppose we are given a dataset $\mathcal D = \{(\x_i,y_i)\}_{i=1}^N$ in general position and some $m \in \N$.
Let $L^*(\mathcal D)$ denote the global minimum of the loss \eqref{eq:def_emp_loss} over the dataset $\mathcal D$ for a shallow ReLU network with $m$ units.
Given any $\epsilon > 0$, some $\delta > 0$ exists such $|L^*(\mathcal D) - L^*(\mathcal D_\epsilon)| < \delta$ for any dataset $\mathcal D_\epsilon = \{(\x'_i,y_i)\}_{i=1}^N$ satisfying $\|\x_i-\x'_i\|_2 \leq \epsilon$.
\end{theorem}
\begin{proof}[Proof sketch]
For a small enough perturbation, we can prove that the datasets' zonotopes are combinatorially equivalent.
Hence their sets of feasible activation patterns will be exactly the same.
Using the fact that any subset of a set in general position is also in general position, we can then show that the constrained convex optimization problem associated with each vertex is continuous with respect to the input dataset.
Since the global minimum of the loss is just the minimum of the optimal loss for each vertex, its continuity follows from the fact that the composition of two continuous functions is continuous.
\end{proof}
When the dataset $\mathcal D$ is not in general position, there are two possible ways in which breaking of nontrivial affine dependencies between examples can cause the global minimum of the loss to become discontinuous.
The first is that the globally optimal vertex in the perturbed zonotope exists in the original zonotope, but its associated constrained convex optimization problem is discontinuous with respect to the dataset.
This can happen when there are nontrivial affine dependencies that get broken amongst the active examples in the vertex.
The second way is that the globally optimal vertex of the perturbed zonotope does not exist in the original zonotope.
Geometrically, we can think of such a vertex as resulting from the breakdown of a non-parallelepiped higher dimension face \citep{gover2014congruence}.
See \cref{sec:breakdown_ngp_examples} for examples of these phenomena.

\paragraph{Analysis of Reductions}

We can use this characterization of the instabilities of the global optimum to perturbations in the training data to analyze the reductions of NP-hard problems used in \citet{goel2020tight}.
We focus on the reduction of the NP-hard set cover problem to the optimization of a single bias-free ReLU.
In \cref{sec:appendix_unstable_data_example}, we provide an explicit example of an arbitrarily small perturbation that results in the global minimum of the loss being independent of the solution to the set cover problem.

To the best of our knowledge, existing reductions of NP-hard problems to ReLU optimization all create datasets that are not in general position \citep{goel2020tight, froese2021computational}.
Therefore, we present a modification of the set cover reduction that produces a dataset in general position.
See \cref{sec:appendix_reduction_general_position} for details of this modification along with a proof that it is indeed a reduction of the set cover problem.
We thus have the following statement.

\begin{theorem}\label{thm:np_hard_gp}
Optimizing a ReLU is NP-hard even when restricted to datasets in general position.
\end{theorem}

\subsection{Polynomial Time Optimization via Overparameterization}\label{sec:overparam_polytime}
Even though ReLU network optimization is NP-hard in general, it can be shown that overparameterization allows for gradient descent to converge to the global minimum in polynomial time \citep{du2018gradient,zou2019improved,oymak2020toward,allen2019convergence}.
This is not a contradiction since optimization of overparameterized ReLU networks is a strict subset of the set of all ReLU optimization problems.

The general proof method of these works usually involves demonstrating that overparameterization results in activation patterns not changing much throughout training.
This allows gradient descent to effectively bypass the combinatorial search of the outer loop in \cref{alg:arora}.
The remaining optimization problem can then be shown to be similar to the constrained convex optimization problem \eqref{eq:activation_region_constrained_opt} by assuming that the second layer is frozen.
%
Using the zonotope formalism, we can interpret these results as saying that a sufficiently large number of hidden units $m$ guarantees with high probability that a randomly chosen vertex corresponds to a region of parameter space containing a global minimum of the loss.
Parameter initialization selects the random vertex in practice.

This can be justified theoretically through a connection to random feature models.
Here we assume that the first layer is frozen, and the second layer forms a linear model over the random first layer features.
As the number of units $m$ increases past the number of training examples $N$, the set of first layer activations can become linearly independent.
The probability of this approaches 1 as $m \to \infty$.
Whether all of the parameters within an activation region produce linearly dependent activations can be shown to depend solely on its activation pattern when the dataset is in general position.

To test this, we ran experiments comparing batch gradient descent to solving \eqref{eq:activation_region_constrained_opt} for a randomly chosen vertex on some toy datasets.
We created synthetic datasets by first choosing the input dimension $d$ and a positive integer $\mgen$.
To get the training examples, we sampled $N = (d + 1) \mgen$ points in $\R^d$ i.i.d.\ from the standard Gaussian distribution.
We then sampled the weights of a shallow ReLU network with $\mgen$ units i.i.d.\ from the standard Gaussian distribution.
We used this network to create the labels for our synthetic dataset.
See \cref{sec:synth_data_gen} for details on the data generation process.

We also created toy binary classification datasets from MNIST \citep{lecun2010mnist} and Fashion MNIST \citep{xiao2017fashion} by choosing two classes, 5/9 and coat/pullover, respectively, to differentiate.
We used the first $d \in \{8, 16\}$ components of the PCA whitened data and selected $N \in \{350, 700\}$ examples for our training sets. 
See \cref{sec:real_world_data_exps} for details.

\begin{figure}[t]
\begin{center}
\centerline{\includegraphics[width=0.45\linewidth]{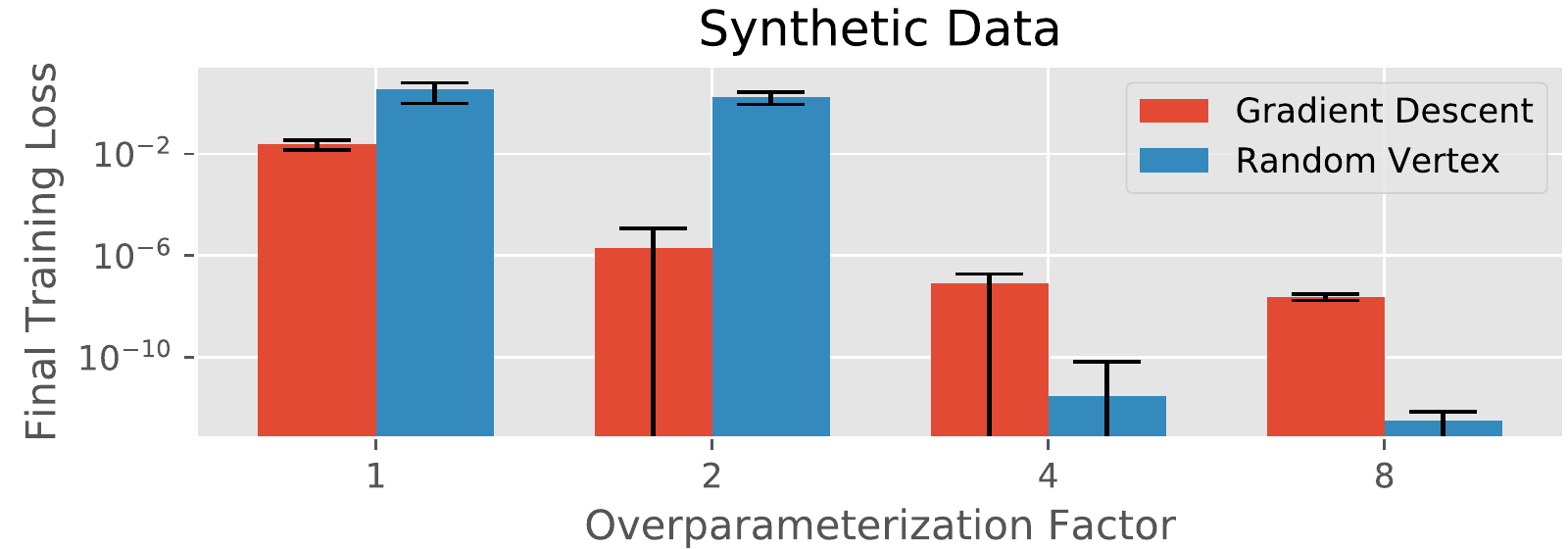}
\includegraphics[width=0.5\linewidth]{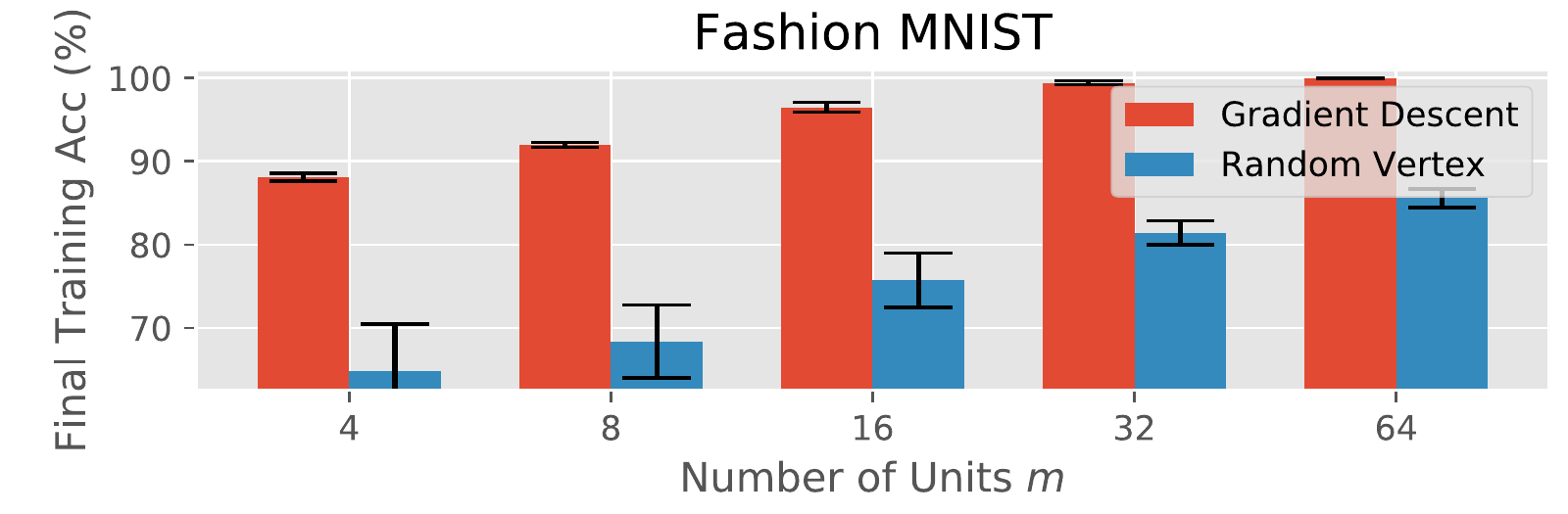}}
\vspace{-0.5em}
\caption{Comparison between gradient descent and optimization with a fixed random activation pattern.
Left: Results for MSE on synthetic data for $d=8$ and $\mgen = 8$. The overparameterization factor times $\mgen$ equals the number of units in the trained network.
Right: Results for accuracy on Fashion MNIST coat/pullover binary classification for $d=16$ and $N=700$.}
\label{fig:gd_vs_rv}
\end{center}
\vskip -0.2in
\end{figure}

We present some of our results in \cref{fig:gd_vs_rv}.
See \cref{sec:exp_details} for details of the training procedures and for results on more $d,\mgen$ and $d,N$ pairs.
On synthetic data, we see that the random vertex method finds a good solution for overparameterization factors of 4 and up.
However, gradient descent tends to arrive at reasonably good solutions for lower levels of overparameterization while the random vertex method fails.
This was a general trend that we observed across different $d,\mgen$ pairs on the synthetic datasets and $d,N$ pairs on the binary classification datasets.
Note that the Fashion MNIST networks represented in \cref{fig:gd_vs_rv} were relatively underparameterized with the maximal size of 64 units being overparameterized by only a factor of about 1.5.

This demonstrates that gradient descent can perform some aspects of the combinatorial search over zonotope vertices.
We hypothesize that the gradient tends to be smaller within activation regions with a good optimum and thus gradient descent is more likely to stay within a good activation region.
Conversely, the larger gradients within activation regions with poor optima make it more likely that a gradient descent step will move the parameters out of those regions.
We can thus think of gradient descent as performing a pseudo-annealing process over the vertices of the zonotope since the likelihood of moving from one vertex to another decreases as the parameters settle into better activation regions.


\subsubsection{Tighter Bounds}\label{sec:tighter_bounds}

We now introduce a novel vertex selection scheme that runs in polynomial time and requires minimal overparameterization.
Suppose $\mathcal D = \{(\x_i,y_i)\}_{i=1}^N$ is a dataset in general position.
Assume that the examples are ordered by the value of their last coordinate, which we suppose is unique WLOG (i.e.\ $\e_d^T\x_i < \e_d^T\x_j$ for $i < j$).
If not provided in this format, this can be accomplished in $O(N\log N)$ time.
We now split the dataset into $\lceil \frac{N}{d+1} \rceil$ chunks containing at most $d+1$ examples.
We write each chunk as $\mathcal D_k = \{(\x_i,y_i)\}_{i=(k - 1)(d+1)+1}^{\min(N, k(d+1))}$.
Since each $\mathcal D_k$ contains a contiguous chunk of examples sorted along an axis in coordinate space, we see that we can always find a hyperplane separating $\mathcal D_k$ and $\mathcal D_{k'}$ for $k\neq k'$.
For each $k = 1,\dotsc, \lceil \frac{N}{d+1} \rceil$, we add two units to our ReLU network and assign them the activation pattern of 0 for examples belong to a $\mathcal D_{k'}$ with $k'<k$ and 1 for the remaining examples.
One of the units will be multiplied by +1 in the second layer while the other will be multiplied by -1.
Hence the network contains a total of $2 \lceil \frac{N}{d+1} \rceil$ hidden units.
We prove in \cref{sec:tighter_bounds_proof} that a set of weights with that activation pattern exists such that the output of the network on training examples exactly matches their labels.
The key idea in the proof is that we can sequentially fit the examples in the $k$-th chunk without undoing our progress in fitting the chunks before it.

\begin{theorem}\label{thm:tighter_op_bounds}
Given a dataset in $\R^d$ containing $N$ examples in general position, a shallow ReLU network containing $2 \lceil \frac{N}{d+1} \rceil$ hidden units can be found in polynomial time exactly fitting the dataset.
\end{theorem}

To the best of our knowledge, this is the tightest known bound on the amount of overparameterization needed to find the global optimum of a ReLU network in polynomial time.
A simple argument comparing the number of unknowns and the number of equations demonstrates that we need at least $\frac{N}{d+1}$ hidden units to exactly fit an arbitrary dataset with a shallow ReLU network.
Hence our method uses only about twice as many hidden units as is necessary to fit the data.
However, we emphasize that this ReLU optimization scheme is primarily of theoretical interest since we find in practice that the resulting ReLU network tends to be a very ill-conditioned function.


\subsection{Relevance to Optimization in Practice}
Practically all optimization of ReLU networks in practice uses some variant of gradient descent with an overparameterized network.
As the degree of overparameterization goes down, gradient descent begins to arrive at increasingly suboptimal solutions \citep{nakkiran2021deep}.

%
In \cref{sec:overparam_polytime}, we hypothesized how gradient descent can find activation regions containing good optima.
However, the gradient of the loss is inherently a local property in parameter space while the space's decomposition into activation regions is inherently global.
Boundaries between regions correspond to discontinuities in the gradient of the loss.
We hypothesize that these properties lead to little direct information about the optimization problem being used to inform gradient descent's traversal over zonotope vertices.
Hence we suspect that algorithms that explicitly traverse zonotope vertices using some loss-based criteria can outperform gradient descent in the underparameterized- to mildly-overparameterized regimes.



\subsubsection{Difficulty of Combinatorial Search}\label{sec:diff_comb_search}
Unless $P=NP$, we are unlikely to find an efficient algorithm to perform the combinatorial search in \cref{alg:arora} for arbitrary datasets \citep{goel2020tight}.
However, this does not preclude the existence of heuristics that tend to work well on problems encountered in practice.
We investigated this by using a greedy local search (GLS) over the graph formed by the zonotope's 1-skeleton.
We start by selecting a vertex at random and find its corresponding optimal loss by solving a convex program.
We iterate over its neighboring vertices and compute their optimal losses as well.
We then move to the neighboring vertex with the lowest loss and repeat the process until we arrive at a vertex with lower loss than its neighbors.
We then take that vertex's optimal parameters as our approximation to the global minimization problem.
This algorithm is defined in detail in \cref{alg:local_search}.

We also experimented with some additional heuristics that help the GLS converge faster by reducing the number of convex problems solved at each step.
For example, we can greedily move to the first neighboring vertex encountered with a lower loss, which significantly decreases the time per step in the early stages of training. We can further improve this by using geometric information about a solution's relative location in its activation region to try certain vertices first.
We call the algorithm with these heuristics modified greedy local search (mGLS) and define it in detail in \cref{sec:mgls}.


\begin{figure}[t]
\begin{minipage}{0.45\textwidth}
\begin{algorithm}[H]
\hsize=\textwidth 
    \caption{Greedy Local Search (GLS) Heuristic}\label{alg:local_search}
\begin{algorithmic}
  \STATE {\bfseries Input:} data $\mathcal D = \{\x_i,y_i\}_{i=1}^N$, 2nd layer $\v \in \R^{m + 1}$, max steps $T \in \N$
  \STATE $A_0 \in \operatorname{vert}(\mathcal Z^m)$ 
  \FOR{$t \in \{0, \dotsc, T\}$}
        \STATE $A_{t+1} \gets A_t$
        \FOR{$A' \in \operatorname{neighbors}(A_t)$}
            \IF{$\mathcal L^*(A'; \mathcal D) < \mathcal L^*(A_{t+1}; \mathcal D)$}
                \STATE $A_{t+1} \gets A'$
            \ENDIF
        \ENDFOR
        \IF{$A_{t+1} = A_t$}
            \STATE \textbf{return} $A_t$
        \ENDIF
  \ENDFOR
\STATE \textbf{return} $A_T$
\end{algorithmic}
\end{algorithm}
\end{minipage}\hfill
\begin{minipage}{0.5\textwidth}
\begin{center}
\centerline{\includegraphics[width=\textwidth]{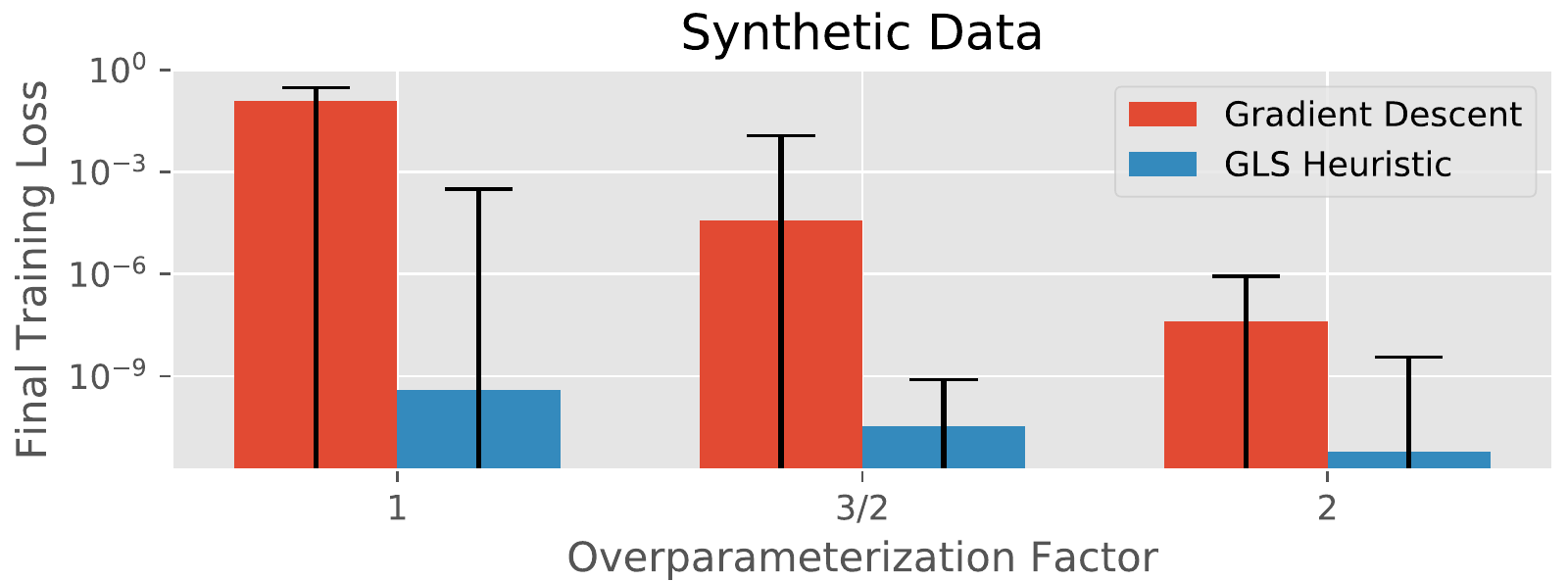}}
\centerline{\includegraphics[width=\textwidth]{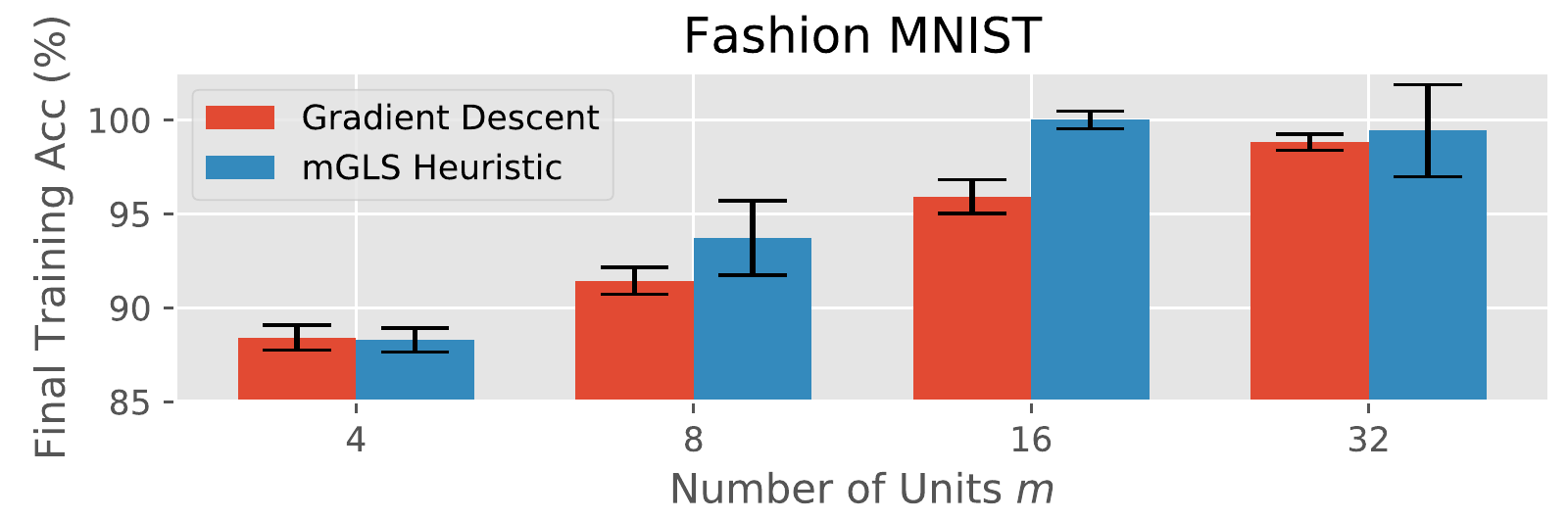}}
\vspace{-0.5em}
\caption{
Comparison between gradient descent and our GLS heuristics.
Top: Results for MSE on synthetic data for $d=4$ and $\mgen = 2$. The overparameterization factor times $\mgen$ equals the number of units in the trained network.
Bottom: Results for accuracy on Fashion MNIST coat/pullover binary classification for $d=8$ and $N=350$.}
\label{fig:gd_vs_gls}
\end{center}
\vskip -0.2in
\end{minipage}
\end{figure}

Note that these heuristics are not guaranteed to return a global minimizer of the loss.
Furthermore, there are an exponential number of vertices in the zonotope, so there are no immediate guarantees of them taking less than exponential time to run.
However, each step takes polynomial time since each vertex has $O(mN)$ neighbors, so each step solves a polynomial number of convex optimization problems.

We ran experiments comparing these heuristics to gradient descent on toy datasets generated in the same way as in \cref{sec:overparam_polytime}.
We used GLS on the synthetic data and mGLS on the MNIST and Fashion MNIST derived data.
We present some of our results in \cref{fig:gd_vs_gls}.
See \cref{sec:exp_details} for details of the training procedures and for results for more values of $d,\mgen$ and $d,N$ pairs.
On synthetic data, the GLS heuristics significantly outperformed gradient descent.
On binary classification tasks, mGLS outperformed gradient descent for networks with moderate levels of underparameterization and performed similarly otherwise.
We observed a similar trend across the rest of the $d,\mgen$ and $d,N$ pairs.
These results support our hypothesis that gradient descent is suboptimal in the combinatorial search over activation patterns.
Furthermore, they suggest that this combinatorial optimization might tend to be tractable in practice.

\section{Related Work}

Some of the concepts in this work also arise in \citet{zhang2018tropical}.
A key difference is that they analyze the activation regions in input space given a ReLU network with fixed parameters.
We can, in fact, use their tropical geometric approach to derive our zonotope formalism for a single ReLU unit.
To do so, the roles of the weights and data must be swapped; we instead use a fixed data matrix and varying vector of weights while they use a fixed weight matrix and varying vector of data.
Our use of a zonotope generated by the training examples, however, is novel.
%
\Citet{misiakos2021neural} show that the approximation error between two shallow ReLU networks depends on the the Hausdorff distance between the zonotopes generated by the each network's units.
\Citet{bach2017breaking} also use a Hausdorff distance between zonotopes in the context of neural network optimization.

\Citet{goel2020tight} provide proofs of the NP-hardness of optimization of shallow ReLU networks and the hardness of even finding an approximate solution.
\Citet{froese2021computational} extend these results and show that the brute force search in \citet{arora2016understanding} cannot be avoided in the worst case.
\Citet{du2018gradient} was one of the first works to prove that overparameterization in shallow ReLU networks allows gradient descent to converge to a global optimum in polynomial time.
Their bound of $\Omega(N^6)$ on the number of hidden units needed for convergence was improved upon by subsequent work \citep{ji2019polylogarithmic, daniely2019neural}.
For example, \citet{oymak2020toward} proved a bound of $\Omega(N^2/d)$.
%

%
\Citet{pilanci2020neural} and \citet{wang2021hidden} represent global optimization of shallow ReLU networks with $\ell_2$ regularization using a convex optimization problem that operates simultaneously over \textit{all} activation patterns for a single unit.
Multiple units are handled by summing over the activations with different activation patterns.
This leads to exponential complexity in the data dimension $d$ but avoids exponential complexity in the number of units $m$.
\Citet{dey2020approximation} provide an example of a heuristic algorithm that searches over activation patterns for a single ReLU unit.
Their algorithm operates on the principle that examples with large positive labels are more likely to belong to the active set in good solutions.

\section{Conclusion}

We introduced a novel characterization of the combinatorial structure of activation patterns implicit in the optimization of shallow ReLU networks.
We showed that it can be described as a Cartesian product of zonotopes generated by the training examples.
%
We used this zonotope formalism to explore aspects of the optimization of shallow ReLU networks.
It provides a natural way to describe instabilities of the global minimum to perturbations of the dataset.
We then related this to work on the NP-hardness of global ReLU optimization.
In particular, we demonstrated that this optimization problem is still NP-hard even when restricted to datasets in general position, which is commonly assumed of data in practice.

We then explored how combinatorial considerations play into the relationship between overparameterization and polynomial-time optimization of shallow ReLU networks.
Namely we interpret known results for gradient descent as stating that a randomly chosen zonotope vertex will be close to one whose activation region contains a good local optimum.
We then provide empirical evidence that sufficient overparameterization makes it highly likely that a randomly chosen vertex has a good local optimum.
We also provide a polynomial-time algorithm that can find a vertex containing the global optimum using approximately twice the minimum number of hidden units needed to fit the dataset exactly.
Finally, we provide a GLS heuristic over zonotope vertices that outperforms gradient descent on some toy problems.
In future work we plan to theoretically and empirically explore heuristics and algorithms that perform well on real-world datasets.
We hope to analyze how vertex choice impacts generalization.
Further insights might be derived by exploring the connections of hyperplane arrangements to tropical geometry and oriented matroids \citep{stanley2004introduction, oxley2006matroid, maclagan2015introduction}.
One caveat of our theory is that it applies to only a shallow ReLU network.
However, the concepts of activation patterns are still meaningful for deep ReLU networks but require real algebraic geometry for analysis \citep{basu2014algorithms, bochnak2013real}.
We hope that further research along these avenues will deepen our understanding of neural network training and enable improvements to training in practice.

\bibliography{refs}
\bibliographystyle{icml2021}

\clearpage
\appendix
\setcounter{table}{0}
\renewcommand{\thetable}{A\arabic{table}}

\section{Mathematical Background}\label{sec:math_background}


\subsection{Oriented Hyperplane Arrangements}
Any $\bar \x \in\R^{d+1}$ defines a tripartite division of $\R^{d+1}$ given by
\begin{equation}\label{eq:def_half_space_plane}
    H^\alpha = \{ \w \in \R^{d + 1} \,\mid\, \sign(\w^T\bar\x) = \alpha \},
\end{equation}
where $\alpha \in \{-1, 0, +1\}$.
The set $H^0$ is a hyperplane while $H^+, H^-$ are the positive and negative open half-spaces, respectively.
We call their closures $\bar H^{\pm} = H^0 \cup H^{\pm}$ the positive/negative closed half-spaces.
For consistency, we say $\bar H^0 = H^0$.

We are provided with a finite set of vectors $\{\bar\x_i\}_{i=1}^N \subseteq \R^{d+1}$.
We interpret these vectors as the training examples in \cref{sec:comb_struct}, but for now consider them to be arbitrary vectors.
Let $H^\alpha_i$ denote the subsets given by \eqref{eq:def_half_space_plane} for the $i$-th example.
The union of the hyperplanes $\bigcup_{i=1}^N H_i^0$ separates chambers of $\R^{d+1}$ into a finite number of disjoint cells.
Each cell can be described as the intersection of open half-spaces and thus can be indexed by a vector of sign patterns reflecting whether the positive or negative half-space is used.
The structure induced by these hyperplanes along with their orientation information creates what is known as an oriented hyperplane arrangement \citep{richter20176}.

\subsection{Polyhedral Complexes}
We can also describe this arrangement through the notion of a polyhedral complex \citep{ziegler2012lectures}.
We first provide some definitions that will be useful.
A polyhedral set is any set that is the intersection of a finite number of closed half-spaces.
A polytope is a bounded polyhedral set.
Any hyperplane intersecting a polyhedral set will either divide it into two polyhedral sets or only intersect it on its boundary.
In the latter case, we call such a hyperplane a supporting hyperplane.
A face of a polyhedral set is defined as its intersection with a supporting hyperplane.
The dimension of a face is the dimension of its affine span.
If the dimension of a face is $k$, we call it a $k$-face.
We call 0-faces vertices and 1-faces edges.
By convention, the empty set is a face of any polyhedral set.
The 1-skeleton of a polyhedral set is the graph formed by its vertices and edges.

A polyhedral complex $\mathcal K$ is a finite set of polyhedral sets satisfying
\begin{enumerate}[noitemsep,topsep=0pt,parsep=0pt,partopsep=0pt]
    \item If $P \in \mathcal K$ and $F$ is a face of $P$, then $F \in \mathcal K$.
    \item If $P_1,P_2\in\mathcal K$, then their intersection $P_1 \cap P_2$ is a face of both $P_1, P_2$.
\end{enumerate}
We call a codimension 0 member of $\mathcal K$ a chamber of the polyhedral complex.
The support of a polyhedral complex is the union of its polyhedral sets.
If a polyhedral complex's support equals the entire space, then we call it a polyhedral decomposition of the space.
The ``is a face of'' relation induces a poset structure on the members of a polyhedral complex, which we call its face poset.
This can be extended further to a meet-semilattice with the meet operation being given by set intersection.
We call this the face semilattice.

Going back to the case of an oriented hyperplane arrangement, let $\a\in\{-1,0,+1\}^N$ be some sign pattern.
Now let us write
\begin{equation}\label{eq:def_sign_pattern_region}
    R^\a = \bigcap_{i=1}^N \bar H_i^{a_i}.
\end{equation}
Since our hyperplanes are all linear (i.e.\ non-affine), we always have $\0 \in R^\a$.
If $R^\a = \{\0\}$, we say that $R^\a$ is null.
Define $\mathcal R$ to be the set of all $R^\a$.
Then $\mathcal R$ is a polyhedral complex, which we prove in \cref{sec:R_poly_complex_proof}.

Every chamber of $\mathcal R$ corresponds to a non-null $R^\a$ with $\a\in\{-1, +1\}^N$.
If the hyperplanes are general positions, this correspondence is one-to-one.
The face semilattice of $\mathcal R$ provides information about how its chambers are arranged in space.
For example, let $M \in \mathcal R$ be the meet of two chambers.
If $M \neq \{\0\}$, then those chambers are neighbors.
Then $\dim M \in \{1, \dotsc, d\}$ and the sign patterns of the chambers differ by at least $d + 1 - \dim M$ sign flips, with equality always holding the hyperplanes are in general positions.

\subsection{Dual Zonotopes}
It turns out that we can describe the incidence structure of the polyhedral complex $\mathcal R$ nicely with a single polytope called its dual zonotope  $\mathcal Z$ \citep{ziegler2012lectures}.
A zonotope is any polytope that can be expressed as the Minkowski sum of a finite set of line segments called its generators \citep{mcmullen1971zonotopes}.
In the case of the dual zonotope of $\mathcal R$, these generators are the line segments $\{[\0,\bar\x_i]\}_{i=1}^N$.
We can thus write 
\begin{equation}\label{eq:def_assoc_zono}
    \mathcal Z = \left\{\sum_{i=1}^N \lambda_i \bar\x_i \mid \lambda_i \in [0, 1]\right\}.
\end{equation}
When the generators are in general positions, each $k$-face of $\mathcal Z$ is a $k$-dimension parallelepiped.
Nontrivial linear dependencies between generators, however, lead to $k$-faces that are the union of multiple $k$-dimension parallelepipeds lying in the same $k$-dimension affine subspace.

The duality between $\mathcal Z$ and $\mathcal R$ allows us to associate members of $\mathcal R$ with faces of $\mathcal Z$.
Each $k$-face of $\mathcal Z$ corresponds to a codimension $k$ member of $\mathcal R$.
Notably, the vertices of $\mathcal{Z}$, denoted by $\operatorname{vert}(\mathcal Z)$, correspond to the chambers of $\mathcal R$.
Relationships between members of $\mathcal R$ carry over to faces of $\mathcal Z$.
For example, two neighboring chambers of $\mathcal R$ whose sign patterns differ by a single flipped sign will correspond to two vertices connected by an edge in $\mathcal Z$.

We now describe how to make this correspondence explicit.
Let $\v\in\operatorname{vert}(\mathcal Z)$ be a vertex.
It can be shown that $\v$ has a unique representation as $\sum_{i=1}^N \lambda_i \bar\x_i$ with every $\lambda_i \in \{0, 1\}$.
We call the vector $(\lambda_1,\dotsc,\lambda_N)$ the barcode of the vertex.
We will often treat a vertex interchangeably with its barcode in this paper with difference being clear by context.
Let $\a$ be the sign pattern of the chamber corresponding to $\v$.
Then $a_i = -1$ if $\lambda_i = 0$ and $a_i = +1$ if $\lambda_i = 1$ for $i=1,\dotsc,N$.

Now suppose two vertices $\v_1,\v_2\in\operatorname{vert}(\mathcal Z)$ are connected by an edge, and that the hyperplanes of $\mathcal R$ are in general positions.
Then there exists a single $i^*$ such that, WLOG, $\v_2 = \v_1 + \bar\x_{i^*}$.
The sign pattern of the member of $\mathcal R$ corresponding to the edge can then be found by finding the sign pattern for $\v_1$ and changing its $i^*$-th entry to be 0.

\subsubsection{Cartesian Power of Zonotopes}\label{sec:cart_prod_zono}
Let us consider an $m$-ary Cartesian power of a zonotope 
$\mathcal Z^m = \prod_{i=1}^m \mathcal Z$.
We can see that $\mathcal Z^m$ is also a zonotope and is generated by line-segments from the origin to members of 
$\bigcup_{i=1}^m \bigcup_{j=1}^N \{\e_i\bar\x_j^T\}$,
where $\e_i \in \R^m$ is the $i$-th standard coordinate vector.
Each $k$-face of $\mathcal Z^m$ is the Cartesian product of a set of $\{k_1,\dotsc,k_m\}$-faces of $\mathcal Z$ where $k = k_1 + \dotsi + k_m$.
Notably, each vertex of $\mathcal Z^m$ corresponds to a product of $m$ vertices of $\mathcal Z$.
Edges of $\mathcal Z^m$ correspond to the product of a single edge of $\mathcal Z$ with $m-1$ vertices.

\section{Proofs for \Cref{sec:math_background}}

\subsection{Proof that $\mathcal R$ is a Polyhedral Complex}\label{sec:R_poly_complex_proof}

Recall that a polyhedral complex $\mathcal K$ is a finite set of polyhedral sets satisfying
\begin{enumerate}
    \item If $P \in \mathcal K$ and $F$ is a face of $P$, then $F \in \mathcal K$.
    \item If $P_1,P_2\in\mathcal K$, then their intersection $P_1 \cap P_2$ is a face of both $P_1, P_2$.
\end{enumerate}

Recall that we have defined $\mathcal R$ as
\begin{equation}
    \mathcal R = \left\{ R^\a \subseteq \R^{d+1} \mid \a\in\{-1, 0, +1\}^N\right\},
\end{equation}
where $R^\a$ is given by \eqref{eq:def_sign_pattern_region}.
Note that generally $|\mathcal R| \leq 3^N$ since multiple $R^\a$ can equal $\{\0\}$.

It is easy to see that every $R^\a \in \mathcal R$ is a polyhedral set since it can be can defined as the intersection of finitely many closed half-spaces.
When $a_i = 0$, its corresponding hyperplane in  \eqref{eq:def_sign_pattern_region} is equivalent to the intersection of its positive and negative closed half-spaces.

We now prove the first condition for $\mathcal R$ being a polyhedral complex.
Suppose $R^\a \in \mathcal R$ and suppose $F$ is a face of $R^\a$.
Hence $F$ is the intersection of $R^\a$ with a supporting hyperplane.
It is straightforward to see that any face of $R^\a$ can be represented by a $R^\b$ where $b_i = a_i$ for all $i \in [N] \setminus I$ and $b_i=0, a_i=\pm 1$ for $i\in I \subseteq [N]$.
Hence the face $F = R^\b \in \mathcal R$.

We now prove the second condition for $\mathcal R$ being a polyhedral complex.
Let $R^\a,R^\b \in \mathcal R$.
From \eqref{eq:def_sign_pattern_region}, we see that
\begin{equation}
    R^\a \cap R^\b = \bigcap_{i=1}^N \bar H_i^{a_i} \cap \bar H_i^{b_i}.
\end{equation}
We see that $H_i^{a_i} \cap \bar H_i^{b_i} = H_i^{\pm 1}$ if $H_i^{a_i} = H_i^{b_i} = H_i^{\pm 1}$, and that $H_i^{a_i} \cap \bar H_i^{b_i} = H_i^0$ otherwise.
It is easy to see that this intersection can be represented as the intersection of a supporting hyperplane with either $R^\a$ or $R^\b$.
Hence their intersection is a mutual face.

\section{Proof of \Cref{thm:stability_gp}}\label{sec:stability_proof}

Suppose we are given a dataset $\mathcal D = \{(\x_i,y_i)\}_{i=1}^N$ in general position.
Given some $\epsilon > 0$, let $\mathcal D_\epsilon = \{(\x'_i,y_i)\}_{i=1}^N$ be any perturbation of $\mathcal D$ such that $\|\x_i-\x'_i\|_2 \leq \epsilon$.

Let $\bar X \in \R^{N \times (d + 1)}$ denote the data matrix in homogeneous coordinates for $\mathcal D$.
Let $\bar X'$ denote the corresponding data matrix for $\mathcal D_\epsilon$.
We see that $\|\bar X - \bar X'\|_F \leq \sqrt{N}\epsilon$.
Hence some matrix $P \in \R^{N \times (d + 1)}$ exists such that $\bar X' = \bar X + P$ and $\|P\|_F \leq \sqrt{N}\epsilon$.

We can interpret the zonotope definition \eqref{eq:def_assoc_zono} as saying that a zonotope is the image of a hypercube under the matrix formed by its generators.
Hence if $\mathcal Z$ is the zonotope of the original dataset, we may write
\begin{equation}
    \mathcal Z = \left\{ \bar X^T \u \in \R^{d+1} \mid \u \in [0, 1]^{N} \right\}.
\end{equation}
Let $\mathcal Z'$ denote the corresponding zonotope for the perturbed dataset.

We now wish to show that the vertices of $\mathcal Z$ are in a one-to-one correspondence with the vertices $\mathcal Z'$ for sufficiently small $\epsilon$ with their sets of vertex barcodes coinciding.
Let $\b \in \{0,1\}^N$ be some binary vector such that $\p = \bar X^T \b$ is a vertex of $\mathcal Z$.
Then we know that some affine hyperplane $H \subseteq\R^{d+1}$ exists such that $H\cap \mathcal Z = \{\p\}$.

Let $\q = \bar X^T \c$, where $\c \in \{0,1\}^N$, be the image of an arbitrary vertex of the hypercube such that $\q \neq \p$.
Note that $\q$ is not necessarily a vertex of $\mathcal Z$.
We thus find some $\delta > 0$ such that the distance from $\q$ to the hyperplane $H$ is greater than $\delta$ for every such $\q$.
Furthermore, all such $\q$ will lie in exactly one of the half-spaces formed by the hyperplane.

Let $H'\subseteq\R^{d+1}$ be the affine hyperplane formed by shifting $H$ by $P^T\b$.
If we let $\p' = \bar X'^T \b$, it is straightforward to see that $\p' \in H'$.
Hence $H'$ intersects the perturbed zonotope $\mathcal Z'$.

Let $\q' = \bar X'^T \c$.
Note that we can write $\p' = \p + P^T\b$ and $\q' = \q + P^T\c$.
Let us now bound $\|P^T\b\|_2$.
We see that $\|\b\|_2 \leq \sqrt{N}$.
Using known relations between matrix norms, we see that $\|P^T\|_2 \leq \|P^T\|_F \leq \sqrt{N}\epsilon$.
Hence $\|P^T\b\|_2 \leq N\epsilon$.
By the exact same logic, we see that $\|P^T\c\|_2 \leq N\epsilon$.

Now suppose that we choose $\epsilon < \frac{\delta}{2N}$.
By the triangle inequality, we can see that $\q'$ can move a distance at most $2 N \epsilon < \delta$ relative to the hyperplane $H'$.
Since the distance from $\q$ to $H$ was greater than $\delta$, we see that every such $\q'$ must lie on the same side of the hyperplane $H'$.
Hence $H'\cap \mathcal Z' = \{\p'\}$, which implies that $\p'$ is a vertex of $\mathcal Z'$.
Hence every vertex barcode of $\mathcal Z$ is a vertex barcode of $\mathcal Z'$.
As $\mathcal D_\epsilon$ will also be in general position for small enough $\epsilon$, we can swap the roles of $\mathcal Z$ and $\mathcal Z'$ in our proof to see that every vertex barcode of $\mathcal Z'$ is a vertex barcode of $\mathcal Z$.
It thus follows that there is a one-to-one correspondence between vertex barcodes of $\mathcal Z$ and $\mathcal Z'$.

By the relationship between zonotope vertices and activation regions shown in \cref{sec:comb_struct}, we have thus proved that the set of convex optimization problems for $\mathcal D$ are a slightly perturbed version of the convex optimization problems for $\mathcal D_\epsilon$.

As any subset of a set of vectors in general position is also in general position, we can see that the solution of each convex optimization problem is continuous with respect to perturbations of the data \citep{agrawal2019differentiable}.
The global minimum of the loss is given by the minimum over the set of per-activation-region local minima.
Hence the global minimum is continuous with respect to the training dataset as it is the composition of two continuous functions.

\section{Examples of Discontinuities at Datasets not in General Position}\label{sec:breakdown_ngp_examples}

This section provides examples of the two sources of discontinuities of the minimal loss of a dataset when it is not in general position.

\subsection{Convex Problem Associated to a Vertex being Discontinuous}\label{sec:breakdown_ngp_zono_same}

Here we provide an example of a dataset whose optimal loss is not continuous with respect to the dataset.
For this dataset, the optimal vertex is the same and exists in both the original and perturbed zonotopes.
This means that its associated convex optimization problem is discontinuous with respect to the dataset.

Let us consider the dataset $\mathcal D \subseteq \R^2 \times \R$ given by
\begin{align*}
    X &= \begin{bmatrix}1&2&3&4&5\\0&0&0&0&0 \end{bmatrix}\\
    Y &= \begin{bmatrix}1&2&2.5&4&5\end{bmatrix}.
\end{align*}
Note that $\mathcal X$ lies entirely within the $x_2 = 0$ hyperplane and thus is not in general position.

Now consider the problem of optimizing a single affine ReLU over $\mathcal D$ with respect to the L1 loss.
We assume a linear second layer and take the ReLU's corresponding second layer weight to be 1.
The zonotope $\mathcal Z$ associated to this optimization problem is presented in \cref{fig:breakdown_ngp_zono_same}.\footnote{Technically this is a slice of the zonotope along the $x_2=0$ plane. The full zonotope is equal to the cylinder $\mathcal Z + \R\e_2$.}
Each vertex has been labeled with its minimal loss in bold and with its set of active example indices.
The vertex with the smallest loss of 0.1 is active on all of the examples.

\begin{figure}[t]
\begin{center}
\centerline{\includegraphics[width=0.8\linewidth]{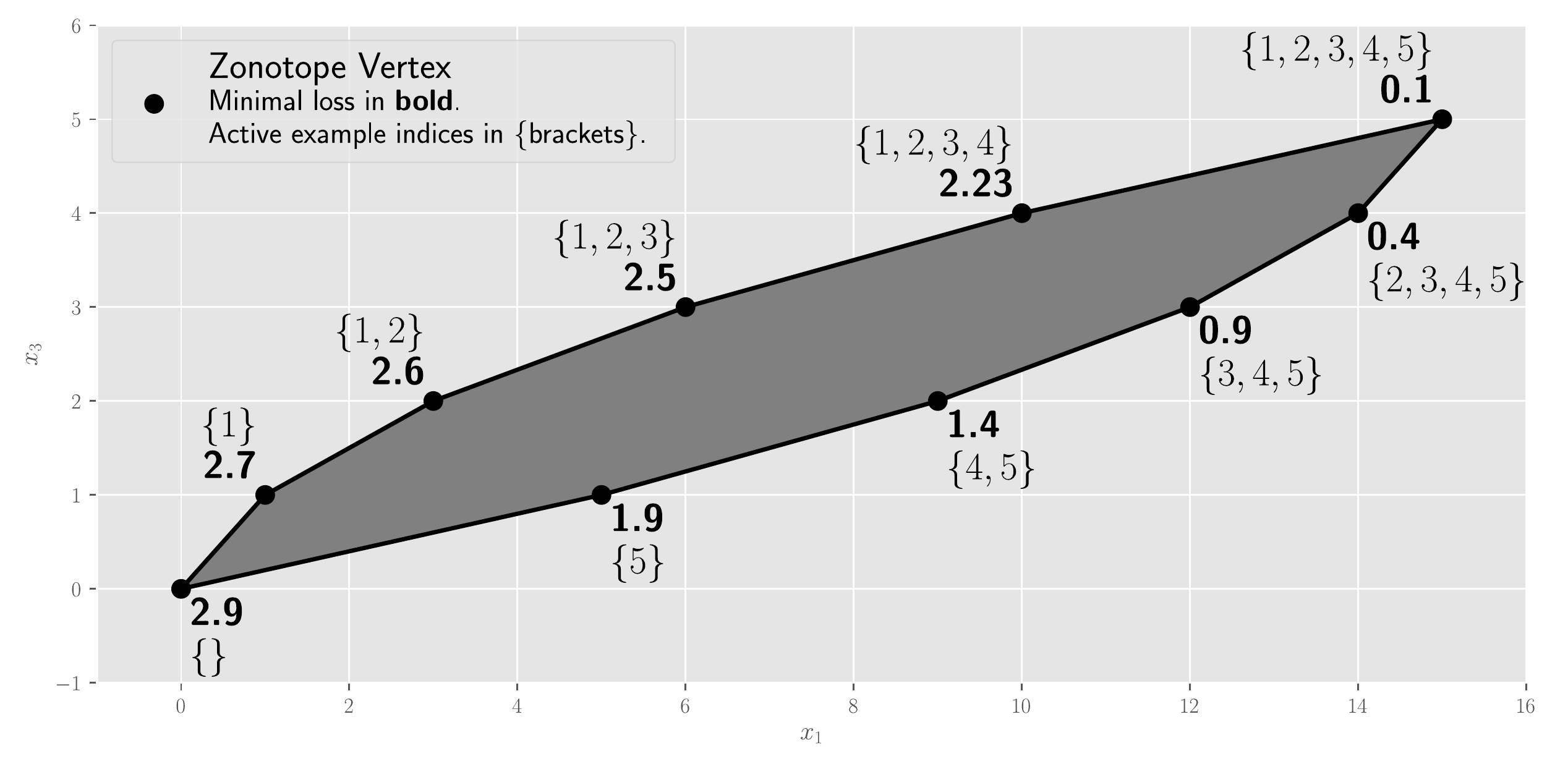}}
\vspace{-0.5em}
\caption{The zonotope $\mathcal Z$ associated to the dataset introduced in \cref{sec:breakdown_ngp_zono_same}.}
\label{fig:breakdown_ngp_zono_same}
\end{center}
\vskip -0.2in
\end{figure}

Now consider what happens when we perform the following perturbation on the dataset
\begin{equation*}
    X_\epsilon = \begin{bmatrix}1&2&3&4&5\\0&0&\epsilon&0&0 \end{bmatrix},
\end{equation*}
where $\epsilon > 0$ is arbitrarily small.
Let $\mathcal D_\epsilon = (X_\epsilon, Y)$.
If we set the parameters $(\w,b)$ of our ReLU to $\w = \begin{bmatrix}1 & -\frac{1}{2\epsilon} \end{bmatrix}^T$ and $b=0$, we see that we fit $\mathcal D_\epsilon$ exactly and thus obtain zero L1 loss.
These parameters belong to the same vertex as the global minimum of the unperturbed dataset.
Hence we conclude that the convex optimization problem associated to this vertex is discontinuous with respect to the dataset.

\subsection{Global Loss in New Vertex of Perturbed Zonotope}\label{sec:breakdown_ngp_zono_new}

Here we provide an example of a dataset whose optimal loss is not continuous with respect to the dataset.
For this dataset, the optimal vertex in the perturbed zonotope does not exist in the original zonotope.

Let us consider the dataset
$\mathcal D \subseteq \R^3 \times \R$ given by
\begin{align*}
    X &= \begin{bmatrix}-1&2&-1&-1\\0&1&1&-1\\0&0&0&0\end{bmatrix}\\
    Y &= \begin{bmatrix}4&3&2&1\end{bmatrix}.
\end{align*}

Consider the problem of optimizing a single linear ReLU over $\mathcal D$ with an L1 loss.
Assume that the second layer is linear with the ReLU's output weight set to 1.
Note that $\mathcal X$ lies entirely within the $x_3 = 0$ plane and thus is not in general position.
The zonotope $\mathcal Z$ associated to this optimization problem is presented in \cref{fig:breakdown_ngp_zono_new_2d}.\footnote{Technically this is a slice of the zonotope along the $x_3=0$ plane. The full zonotope is equal to the cylinder $\mathcal Z + \R\e_3$.}
Each vertex has been labeled with its minimal loss in bold and with its set of active example indices.
The vertex with the smallest loss of 1.25 has the examples $\{\x_2,\x_3\}$ active.

\begin{figure}[t]
\begin{center}
\centerline{\includegraphics[width=0.8\linewidth]{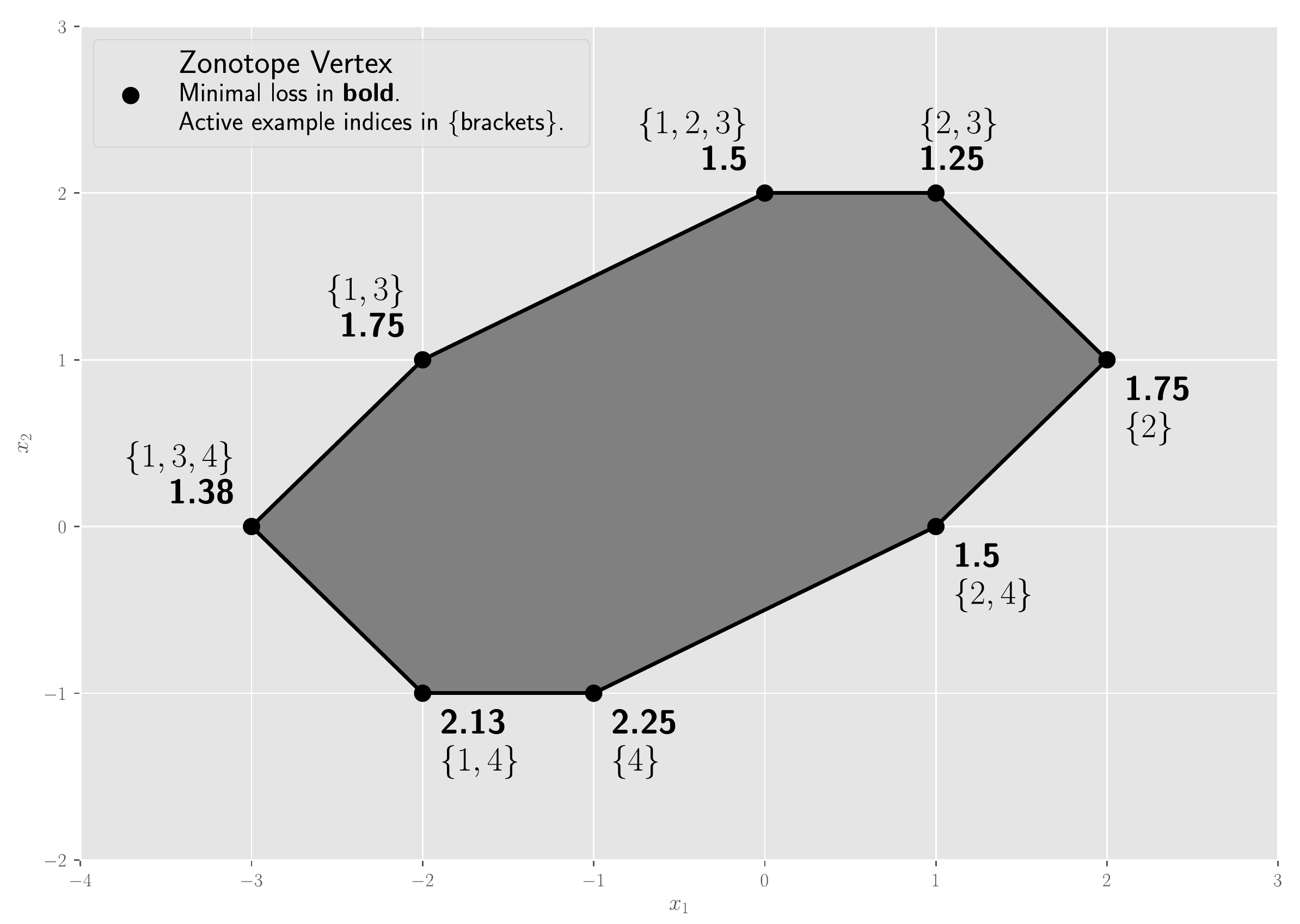}}
\vspace{-0.5em}
\caption{The zonotope $\mathcal Z$ associated to the dataset introduced in \cref{sec:breakdown_ngp_zono_new}.}
\label{fig:breakdown_ngp_zono_new_2d}
\end{center}
\vskip -0.2in
\end{figure}

Now consider what happens when we perform the following perturbation on the dataset
\begin{equation}
    X_\epsilon = \begin{bmatrix}-1&2&-1&-1\\0&1&1&-1\\0&\epsilon&0&0\end{bmatrix}
\end{equation}
where $\epsilon > 0$ is arbitrarily small.
Let $\mathcal D_\epsilon = (X_\epsilon, Y)$.
The global minimum of the loss for $\mathcal D_\epsilon$ occurs at the parameter value $\w = \begin{bmatrix}-\frac32 & \frac12 & \frac{11}{2\epsilon} \end{bmatrix}^T$.
The loss value at this point is 0.625, and these parameters are associated to the zonotope vertex with all examples active.
As evident from \cref{fig:breakdown_ngp_zono_new_2d}, such a vertex does not exist in the unperturbed zonotope.

\section{NP-Hardness}\label{sec:appendix_np_hard}

\Citet{goel2020tight} prove the NP-hardness of optimizing a single ReLU by reducing solving an instance of the NP-hard set cover problem to optimizing a single ReLU over a train dataset.
In the set cover problem, we are given a collection $\mathcal T = \{T_1, \dotsc, T_M\}$ of subsets of a given set $U$.
Given some $t \in \N$, the goal is determine whether a subcollection $\mathcal S \subseteq \mathcal T$ exists such that $U = \bigcup_{S\in\mathcal S} S$ and $|\mathcal S| \leq t$.

\subsection{Reduction to ReLU Optimization}\label{sec:single_relu_og_np_hard}
\Citet{goel2020tight} use a single ReLU without a bias as their model, so we may write our network as
\begin{equation}
    f_\w(\x) = \phi(\w^T\x).
\end{equation}
The input dimension of their model is $d = M + 2$.
Of these dimensions, $M$ correspond to members of $\mathcal T$ and two are used as ``constraint coordinates''.
We use $\e_{T_i}$ to denote a unit coordinate vector corresponding to $T_i$, and $\e_\gamma$ and $\e_1$ as the unit coordinate vectors for the constraint coordinates.

Set $\gamma = 0.01/M^2$.
Overall, they create $N = |U| + M + 2$ labeled training examples.
For the constraint coordinates, they create the examples
\begin{equation}\label{eq:def_xy_gamma}
    (\x_\gamma, y_\gamma) = (\e_\gamma, \gamma)
\end{equation}
and
\begin{equation}\label{eq:def_xy_1}
    (\x_1, y_1) = (\e_1, 1).
\end{equation}
For each $T_i \in \mathcal T$, they create the example
\begin{equation}\label{eq:def_xy_T}
    (\x_{T_i}, y_{T_i}) = (\e_\gamma + \e_{T_i}, \gamma).
\end{equation}
For each $u \in U$, they create the example
\begin{equation}\label{eq:def_xy_u}
    (\x_u, y_u) = (\e_1 + \sum_{T_i \ni u}\e_{T_i}, 0).
\end{equation}
Let $\mathcal D$ denote the entire labeled dataset, and let $\mathcal X$ denote the just the examples without labels.
Note that both $\mathcal D$ and $\mathcal X$ will generally be multisets since if $u,u'\in U$ belong to exactly the same set of subsets in $\mathcal T$, then $\x_u = \x_{u'}$.

Using mean squared error, the training loss of the network can be written as
\begin{equation}
    L(\w) = \frac{1}{N}\sum_{(\x,y)\in\mathcal D}(f_\w(\x) - y)^2.
\end{equation}
\Citet{goel2020tight} show that if the union of $t$ or fewer members of $\mathcal T$ equals $U$, then the global minimum of this loss will be less than or equal to $t\gamma^2/N$.
The weights $\w\in\R^d$ corresponding to this optima will have $w_\gamma = \gamma$, $w_1 = 1$, $w_{T_i} = -1$ for all $T_i \in \mathcal S$, and all other parameters set to zero.

The activation pattern for this optima will have $\x_1$ and $\x_\gamma$ being active while $\x_u$ is inactive for every $u\in U$.
Any $\x_{T_i}$ will be inactive if $T_i \in \mathcal S$ and active otherwise.

\subsection{Discontinuous Response to Perturbation}\label{sec:appendix_unstable_data_example}
Here, we demonstrate how to create a perturbed dataset $\mathcal D'_\epsilon$ so that the the global minimum of loss will always be at most $\gamma^2/N$ regardless of the solution to the set cover problem.
This will hold true even as the scale of the perturbation $\epsilon > 0$ approaches 0.

In this perturbation, we pick an arbitrary $T \in \mathcal T$ and set $\x_u' = \x_u + \epsilon \e_T$ for the examples corresponding to all $u\in U$.
All other examples are left unchanged.

Clearly, $\|\x-\x'\|_2\leq \epsilon$ for all original-perturbed example pairs.
To find a $\w$ with a loss value of $\gamma^2/N$, set $w_\gamma = \gamma$, $w_1 = 1$, $w_T \leq -\epsilon^{-1}$, and set all other parameters to 0.
In this case, the model's predictions are correct for every training example except $(\x_T,y_T)$.
In this case, the model predicts 0 while the label is $\gamma$, so the total loss is $\gamma^2/N$.

Note that the L2 norm of the parameters at the global minimum for the perturbed dataset approaches infinity as the size of the perturbation $\epsilon$ approaches 0.

The activation pattern at this optima will be the same as the activation pattern at the unperturbed optima except that $\x_T$ will be inactive while the rest of the $\x_{T_i}$ will be active.

We note that this activation pattern is achievable on the unperturbed dataset.
However, it requires setting $w_1$ to a small positive value, $w_\gamma$ to a relatively large positive value, all $w_{T_i}$ for $T_i\neq T$ to moderate negative values, and $w_T$ to a large enough negative value.
For example, setting $w_1 = \gamma$, $w_\gamma = 2$, $w_{T_i} = -1$ for $T_i \neq T$, and $w_T = -3$ works.
Hence as discussed in \cref{sec:discont_global_opt}, this discontinuity corresponds to a discontinuity in the constrained convex optimization problem \eqref{eq:activation_region_constrained_opt} associated to a vertex rather than the new optimum occurring at a vertex not present in the original zonotope.

\subsection{Reduction to Dataset in General Position (Proof of \Cref{thm:np_hard_gp})}\label{sec:appendix_reduction_general_position}

It is possible to perturb some of the examples in \cref{sec:single_relu_og_np_hard} to get a dataset in general position that is still the reduction of the subset-sum problem.

\begin{theorem}\label{thm:gen_pos_np_hard}
Let $\delta_1,\delta_2 \in \R$ be constants satisfying  $0 < \delta_1 < \delta_2 < \frac{1}{2d}$.
For each $u \in U$, replace $\x_u$ in the dataset $\mathcal D$ with let $\x'_u = \x_u - \etabf_u$, where $\etabf_u \in \R^d$ is noise sampled IID from the uniform distribution on $[\delta_1, \delta_2]^d$.
Denote this updated dataset as $\mathcal D'$ and let $\mathcal X'$ denote its examples without labels.
Then $\mathcal X'$ is in general linear position with probability 1, and the global minimum of
\begin{equation}
    L'(\w) = \frac{1}{N}\sum_{(\x,y)\in\mathcal D}(f_\w(\x) - y)^2
\end{equation}
is less than or equal to $t\gamma^2/N$ if and only if a set cover of $\mathcal T$ exists containing $t$ sets.
\end{theorem}

The rest of this section is devoted to the proof of \cref{thm:gen_pos_np_hard}.
We first prove that $\mathcal X'$ is indeed in general position.
We then prove both directions of the if and only if statement.

\subsubsection{General Position of Perturbed Dataset}

\begin{lemma}
The examples of the perturbed dataset $\mathcal X'$ are in general linear position.
\end{lemma}
\begin{proof}
We examine linear rather than affine dependencies between examples since the ReLU that we are using has no bias.
Let us partition the training examples as
\begin{equation}
    \mathcal X' = \mathcal X'_U \cup \mathcal X_{\mathcal T} \cup \{\x_1, \x_\gamma\}.
\end{equation}
where $\mathcal X_{\mathcal T}$ corresponds the examples \eqref{eq:def_xy_T} and $\mathcal X'_U$ to corresponds to the examples in \eqref{eq:def_xy_u} with perturbations as in \cref{thm:gen_pos_np_hard}.

From their definitions, it is clear that the set of $N$ vectors $\mathcal X_{\mathcal T} \cup \{\x_1, \x_\gamma \}$ are linearly independent and thus in general position.
Since the $\{\etabf_u\}_{u \in U}$ are sampled IID from the uniform distribution on $[\delta_1, \delta_2]^d$, the introduction of the examples $\mathcal X'_U$ almost surely introduces no new nontrivial linear dependencies.
\end{proof}

\subsubsection{Set Cover of Required Size Exists}
Now suppose that a set cover $\mathcal S \subseteq \mathcal T$  of size $t$ exists.
Choose parameters $\w \in \R^d$ with coordinates $w_1 = 1$, $w_\gamma = \gamma$, $w_{T_i} = -2$ if $T_i \in \mathcal S$, and $w_{T_i} = 0$ otherwise.

\begin{lemma}\label{lem:set_cover_perturbed_satisfies}
The loss on the perturbed dataset is equal to $t\gamma^2/N$ when the parameters are set to $\w$.
\end{lemma}
\begin{proof}
We start by looking at the examples that are unchanged from the original dataset $\mathcal D$ in our perturbed version $\mathcal D'$.
We see that $f_\w(\x_\gamma) = w_\gamma = \gamma = y_\gamma$ and $f_\w(\x_1) = w_1 = 1 = y_1$, so these two examples have a loss of zero.
When $T_i \in \mathcal S$, we have $f_\w(\x_{T_i}) = \phi(w_\gamma + w_{T_i}) = \phi(\gamma - 2) = 0$ since $\gamma < 2$, so these examples incur a loss of $y_{T_i}^2 / N = \gamma^2 / N$.
When $T_i \notin \mathcal S$, we have $f_\w(\x_{T_i}) = \phi(w_\gamma + w_{T_i}) = \phi(\gamma) = \gamma = y_{T_i}$, so these examples have a loss of zero.
Overall these examples contribute a total of $t\gamma^2/N$ to the loss.

Now consider the examples $\x'_u \in \mathcal X'_U$.
By definition, $\x'_u = \x_u - \etabf_u = \e_1 + \sum_{T_i \ni u}\e_{T_i} - \etabf_u$.
The preactivation for such an example is
\begin{align*}
    \w^T\x'_u &= w_1 + \sum_{T_i \ni u} w_{T_i} - \w^T\etabf_u \\
    &= 1 - \sum_{u \in T_i \in \mathcal S} w_{T_i} 2  - \w^T\etabf_u.
\end{align*}
Since $\mathcal S$ is a set cover of $U$, we know at least one $T_i \in \mathcal S$ exists such that $u \in T_i$.
Hence $1 - \sum_{u \in T_i \in \mathcal S} w_{T_i} 2 \leq -1$.
Recall that the entries of $\etabf_u$ are all positive and less than $\delta_2$, and recall that all entries of $\w$ are greater than or equal to $-2$.
Thus, $-\w^T\etabf_u \leq 2d\delta_2$.
Because we defined $\delta_2$ such that $\delta_2 < \frac{1}{2d}$, we see that $2d\delta<1$.
Hence $-\w^T\etabf_u < 1$, so we see that the preactivation is less than $-1 + 1 = 0$.
Thus applying a ReLU activation to this preactivation will output a 0.
Thus $f_\w(\x'_u) = 0 = y'_u$ for all $u \in U$.
Hence the examples from $\mathcal X'_U$ do not contribute to the loss, and the loss at $\w$ is $t\gamma^2/N$.
\end{proof}

\subsubsection{Set Cover of Required Size Does Not Exist}
Throughout this section, let $\w \in \R^d$ be parameter values as defined in the previous section.
We first prove the following lemma.

\begin{lemma}\label{lem:og_loss_implies_perturbed_loss}
If the minimal loss over the original dataset $\mathcal D$ is less than or equal to $t\gamma^2/N$, then the minimal loss over the perturbed dataset $\mathcal D'$ is less than or equal to $t\gamma^2/N$.
\end{lemma}
\begin{proof}
From the proof of Theorem 8 in \citet{goel2020tight}, we know that the minimal loss over the original dataset is less than or equal to $t\gamma^2/N$ only if a set cover of size $t$ exists.
This lemma then follows directly from \cref{lem:set_cover_perturbed_satisfies}.
\end{proof}

We now have the following.

\begin{lemma}\label{lem:perturbed_loss_implies_set_cover}
If the minimal loss over the perturbed dataset $\mathcal D'$ is greater than $t\gamma^2/N$, then no set cover of $U$ exists that is comprised of $t$ or fewer sets from $\mathcal T$.
\end{lemma}
\begin{proof}
The contraposition of \cref{lem:og_loss_implies_perturbed_loss}
states that if the the minimal loss over the perturbed dataset $\mathcal D'$ is greater than $t\gamma^2/N$, then the minimal loss over the original dataset $\mathcal D$ is greater than $t\gamma^2/N$.
From the proof of Theorem 8 in \citet{goel2020tight}, this implies that no set cover of $U$ exists that consists of $t$ or fewer sets from $\mathcal T$.
\end{proof}

\section{Proof for Upper Bound on Required Overparameterization}\label{sec:tighter_bounds_proof}
In this section, we provide a proof of \cref{thm:tighter_op_bounds}.
We begin with the following lemma.

\begin{lemma}\label{lem:tighter_bounds_step_fit}
Let $A,B\subseteq \R^d$ be finite subsets such that $|B| = d+1$ and the $d$-th coordinate of any $\a\in A$ is strictly less than the $d$-th coordinate of any $\b \in B$.
Furthermore, assume that their union $A \cup B$ is in general position.
For any $\b \in B$, let $y_\b\in\R$ be its label.
Then there exist parameters $\w_1,\w_2 \in \R^{d+1}$ satisfying
\begin{align*}
    & \w_1^T\bar\a \leq 0, & \w_1^T\bar\b \geq 0, & \\
    & \w_2^T\bar\a \leq 0, & \w_2^T\bar\b \geq 0, & \\
\end{align*}
for all $\a\in A$ and $\b\in B$
such that
the function
\begin{equation}\label{eq:two_relu_pm}
    f(\x) = \phi(\w_1^T\bar\x) - \phi(\w_2^T\bar\x)
\end{equation}
satisfies $f(\a) = 0$ for all $\a \in A$ and $f(\b) = y_\b$ for all $\b\in B$.
\end{lemma}
\begin{proof}
Let $\alpha \in \R$ be some value that is strictly greater than the $d$-th coordinate of any $\a \in A$ and strictly less than the $d$-th coordinate of $\b\in B$.
Such an $\alpha$ is guaranteed to exist based on the assumptions of the lemma.
Now let $\u = \e_d - \alpha \e_{d+1}\in\R^{d+1}$.
It is clear that $\u^T\bar\a < 0$ for all  $\a \in A$ and $\u^T\bar\b > 0$ for all $\b\in B$.
We can thus multiply $\u$ by a positive scalar to get a vector $\Tilde\u\in\R^{d+1}$ such that  $\Tilde\u^T\bar\a < -1$ for all  $\a \in A$ and $\Tilde\u^T\bar\b > 1$ for all $\b\in B$.

Let us now look at the unconstrained problem of finding a $\w\in\R^{d+1}$ such that $\w^T\bar\b = y_\b$ for all $\b \in B$.
As $B$ contains $d+1$ examples in general positions, such a $\w$ will always exist and can be found via standard linear regression.

Define
\begin{equation}
    \beta_A = \max_{\a\in A} \phi(\w^T\bar\a)
\end{equation}
and
\begin{equation}
    \beta_B = \max_{\b\in B} \phi(-\w^T\bar\b).
\end{equation}
Let $\beta = \max\{\beta_A, \beta_B\}$.
Then setting $\w_1 = \w + \beta\Tilde\u$ and $\w_2 = \beta\Tilde\u$ satisfies the conditions of the lemma.
\end{proof}

We are now ready for the proof.
Let $\mathcal D$ and $\mathcal D_k$ for $k = 1, \dotsc, \lceil \frac{N}{d+1} \rceil$ be defined as in \cref{sec:tighter_bounds}.
Let us do a proof by induction on $k$.
Suppose that $f$ is a ReLU network with $2\lceil \frac{N}{d+1} \rceil - 2$ hidden units fitting the labels in $\bigcup_{k'=1}^{k-1} \mathcal D_{k'}$ exactly.

Let us now relabel the entire dataset by subtracting the predictions of $f$ from the labels to get
\begin{equation}
    \mathcal D' = \left\{(\x, y - f(\x)) \mid (\x, y) \in \mathcal D  \right\}.
\end{equation}
Define $\mathcal D'_k$ accordingly.
Clearly, we have the labels being all zero for all examples in $\mathcal D' \setminus \mathcal D_k'$.
The labels for examples in $\mathcal D_k'$ will generally be non-zero.

We can use \cref{lem:tighter_bounds_step_fit} to find a unit layer network $g$ such that $g(\x) = 0 = y$ for all $(\x, y) \in \mathcal D' \setminus \mathcal D_k'$ and $g(\x) = y$ for all $(\x, y) \in \mathcal D_k'$.
Hence $g$ fits $\mathcal D'$ exactly.
From this it is clear that $f+g$ fits the original dataset $\mathcal D$ exactly.
We can find a ReLU network with $2\lceil \frac{N}{d+1} \rceil$ hidden units representing $f+g$ by having its last two units be the units of $g$ and the remaining units be the units of $f$.

\section{Modified Greedy Local Search}\label{sec:mgls}

\begin{algorithm}[tb]
    \caption{Modified Greedy Local Search (mGLS) Heuristic}\label{alg:mod_local_search}
\begin{algorithmic}
  \STATE {\bfseries Input:} data $\mathcal D = \{\x_i,y_i\}_{i=1}^N$, output weights $\v \in \R^{m + 1}$, max steps $T \in \N$
  \STATE $A_0 \in \operatorname{vert}(\mathcal Z^m)$ \COMMENT{Random initial zonotope vertex.}
  \FOR{$t \in \{0, \dotsc, T\}$}
        \STATE $A_{t+1} \gets A_t$
        \STATE $W^* \gets $ solution of \eqref{eq:activation_region_constrained_opt} for $A_t$ 
        \IF {any of the constraints in \eqref{eq:activation_region_constrained_opt} are equalities}
            \STATE $A_f \gets$ $A_t$ with all those constraints flipped
            \STATE $N_f \gets$ subset of neighbors of $A_t$ differing only on one of those constraints
            \STATE $\mathcal N \gets (\{A_f\}, N_f, \operatorname{neighbors}(A_t) \setminus N_f)$
        \ELSE
            \STATE $\mathcal N \gets (\operatorname{neighbors}(A_t))$
        \ENDIF
        \FOR{$G \in \mathcal N$}
            \FOR{$A' \in G$}
                \IF{$\mathcal L^*(A'; \mathcal D) < \mathcal L^*(A_{t+1}; \mathcal D)$}
                \STATE $A_{t+1} \gets A'$
                \STATE \textbf{continue} main loop
            \ENDIF
            \ENDFOR
        \ENDFOR
        \IF{$A_{t+1} = A_t$}
            \STATE \textbf{return} $A_t$
        \ENDIF
  \ENDFOR
\STATE \textbf{return} $A_T$
\end{algorithmic}
\end{algorithm}

This section provides more information on the additional heuristics used in the mGLS algorithm introduced in \cref{sec:diff_comb_search}.
The purpose of these modifications is to reduce the typical number of convex problems that we have to solve in a run of the algorithm.

The major difference is that as we iterate over neighboring zonotope vertices, we move to any vertex with a lower loss than the current vertex.
This is in contrast to \cref{alg:local_search}, which evaluates the loss at every neighbor and moves to the one with the lowest loss.
Especially near the start of the optimization procedure, we find that this greatly reduces the number of vertices that we need to solve convex programs for.
The order in which we iterate over the neighbors is mostly random with the caveat discussed below.

We also make use of geometric information coming the optimal parameter values given the current vertex to preferentially try some subsets of neighboring vertices first.
If they lie at the boundary of the current activation region, then it stands to reason that activation regions on the other side of that boundary are more likely to have better solutions.
Solutions lying on the boundary of an activation region have a subset of preactivations that are exactly zero.
Equivalently, a subset of the inequalities in \eqref{eq:activation_region_constrained_opt} become equalities at the solution.
In such cases, we first try the vertex that has all of those constraints flipped.
Note that this vertex is not usually a neighbor of the current current in the 1-skeleton of the zonotope and might not even be feasible.
If feasible and $k$ constraints are flipped, then that vertex and the current vertex belong to the same $k$-face of the zonotope.
We then try the neighbors of the current vertex that correspond to flipping one of those constraints.
Afterwards, we try the remaining neighbors.

\section{Experimental Details}\label{sec:exp_details}

\subsection{Synthetic Data}

\subsubsection{Synthetic Dataset Generation}\label{sec:synth_data_gen}

Here we present the details of the generation of the synthetic datasets used in the experiments in this paper.

We start out with the dimension of the input $d$ and the number of units $\mgen$ in the shallow ReLU network used to generate the labels.
We use this to calculate the number of examples $N = (d+1)\mgen$.
We then generate the examples $\{\x_i\}_{i=1}^N \subseteq \R^d$ by sampling them i.i.d.\ from a standard normal distribution.

We use a randomly generated ReLU network to label these examples.
We can express this network as
\begin{equation}
    g(\x) = \v_\text{gen}^T\phi(W_\text{gen}\bar\x) + c_\text{gen}.
\end{equation}
We generate the parameters $\v_\text{gen}\in\R^{\mgen}$, $W_\text{gen}\in\R^{\mgen\times (d+1)}$, and $c_\text{gen}\in\R$ by via sampling from standard normal distributions.
The label for the $i$-th example can then be expressed as $y_i = g(\x_i)$.

\subsubsection{Training Details}\label{sec:exp_training_details}

All experiments on the synthetic datasets used the mean squared error (MSE) loss.
The network architecture for these experiments takes the form of
\begin{equation}
    f(\x) = \v^T\phi(W\bar\x) + c,
\end{equation}
where $\v\in\R^m$, $W\in\R^{m\times (d+1)}$, and $c\in\R$.
We use $m\in\N$ to denote the number of units in the network that we train on the dataset.

\paragraph{Gradient Descent}

All experiments involving gradient descent on synthetic datasets in this paper used batch gradient descent with a learning rate of 1e-3 for 400,000 steps.
All parameters, including the second layer weights $\v,c$, were trained.
We used the parameter initialization scheme from \citet{glorot2010understanding}.

\paragraph{Random Vertex}

A random vertex was selected by sampling the first layer weights $W\in\R^{m\times (d+1)}$ from a standard Gaussian and taking its corresponding activation pattern over the dataset.
We randomly initialize $\v$ with values chosen uniformly from the set $\{-1, 1\}$.
Optimizing over all of the parameters of a shallow ReLU network within a single activation region is non-convex.
However, the problems of training $W,c$ given a fixed $\v$ and training $\v, c$ given a fixed $W$ are convex.
The former is a slight variant of \eqref{eq:activation_region_constrained_opt} while the latter is simple linear regression over fixed features.
We thus iterate between solving these two problems until we converge to fixed loss value.

This process is guaranteed to converge to a local minima of the loss \citep{xu2013block}.
However, it is possible that our process of optimizing the second layer weights here is suboptimal and does not reach the global optimum within the activation region.

\paragraph{GLS Heuristic}

Here we fix $\v\in\R^m$ to have $m/2$ entries set to -1 and $m/2$ entries set to +1.
We slightly modify the convex problem \eqref{eq:activation_region_constrained_opt} to include optimizing the $c\in\R$ in addition to the first layer weights $W\in\R^{m\times(d+1)}$.
We choose the starting vertex at random by sampling the first layer weights $W$ from a standard Gaussian and taking its corresponding activation pattern over the dataset.
We set $T=1024$ as the maximum number of steps.

\subsubsection{Full Results}

We experimented with a range of $d$, $\mgen$, and $m$ values.
Present our full results in \cref{table:full_exp_results}.
The scores represent the median of 16 runs for random vertex scores and the median of 8 runs for the rest.

\begin{table*}[t]
\caption{Results of all synthetic data experiments performed in this paper. The large numbers are the median final MSE over 16 runs for GLS heuristic and over 8 runs for gradient descent and random vertex.
The subscript numbers provide the standard deviation over the runs.
Some cells are empty as the particular combination of $d, \mgen, m,$ and optimization method was not needed for our set of comparisons.}
\label{table:full_exp_results}
\vskip 0.15in
\begin{center}
\begin{small}
\begin{sc}
\begin{tabular}{cccccc}
\toprule
$d$ & $\mgen$ & $m$ & Gradient Descent & Random Vertex & GLS Heuristic \\
\midrule
4 & 2 & 2 & $\text{3.82e-10}_\text{3.2e-04}$ & $\text{1.25e-01}_\text{1.7e-01}$ & $\text{8.27e-01}_\text{1.4e00}$ \\
4 & 2 & 3 & $\text{3.43e-11}_\text{7.6e-10}$ & $\text{3.85e-05}_\text{1.2e-02}$ & -- \\
4 & 2 & 4 & $\text{6.01e-12}_\text{3.7e-09}$ & $\text{4.10e-08}_\text{8.1e-07}$ & $\text{1.58e-01}_\text{5.3e-01}$ \\
4 & 2 & 8 & -- & $\text{8.81e-09}_\text{1.4e-08}$ & $\text{8.29e-13}_\text{1.7e-02}$ \\
4 & 2 & 16 & -- & $\text{1.13e-08}_\text{7.3e-08}$ & $\text{4.11e-17}_\text{1.2e-16}$ \vspace{0.5em} \\
4 & 4 & 4 & $\text{8.30e-03}_\text{1.8e-01}$ & $\text{1.98e-03}_\text{6.6e-02}$ & $\text{8.11e-01}_\text{8.5e-01}$ \\
4 & 4 & 6 & $\text{5.04e-12}_\text{5.7e-10}$ & $\text{3.13e-04}_\text{7.7e-04}$ & -- \\
4 & 4 & 8 & $\text{6.54e-12}_\text{9.7e-10}$ & $\text{4.70e-05}_\text{5.3e-05}$ & $\text{8.47e-09}_\text{7.8e-02}$ \\
4 & 4 & 16 & -- & $\text{5.08e-07}_\text{2.8e-05}$ & $\text{1.68e-12}_\text{1.6e-11}$ \\
4 & 4 & 32 & -- & $\text{2.71e-07}_\text{5.1e-06}$ & $\text{1.97e-15}_\text{4.7e-15}$ \vspace{0.5em} \\
4 & 8 & 8 & $\text{8.53e-03}_\text{5.3e-03}$ & $\text{8.47e-03}_\text{8.3e-03}$ & $\text{4.68e-01}_\text{4.0e-01}$ \\
4 & 8 & 12 & $\text{2.52e-11}_\text{4.3e-10}$ & $\text{7.59e-04}_\text{9.8e-04}$ & -- \\
4 & 8 & 16 & $\text{2.21e-11}_\text{2.6e-09}$ & $\text{6.07e-04}_\text{6.9e-04}$ & $\text{9.96e-03}_\text{3.1e-02}$ \\
4 & 8 & 32 & -- & $\text{1.70e-05}_\text{7.9e-05}$ & $\text{6.68e-12}_\text{2.9e-10}$ \\
4 & 8 & 64 & -- & $\text{6.26e-06}_\text{4.8e-05}$ & $\text{3.49e-14}_\text{2.8e-14}$ \\
\midrule
8 & 4 & 4 & $\text{2.54e-03}_\text{1.2e-01}$ & $\text{1.44e-02}_\text{4.4e-02}$ & $\text{3.74e00}_\text{2.0e00}$ \\
8 & 4 & 6 & $\text{6.11e-11}_\text{1.4e-02}$ & $\text{1.11e-03}_\text{2.6e-02}$ & -- \\
8 & 4 & 8 & $\text{7.69e-12}_\text{8.2e-02}$ & $\text{6.96e-08}_\text{4.2e-05}$ & $\text{4.94e-01}_\text{4.7e-01}$ \\
8 & 4 & 16 & -- & $\text{2.24e-08}_\text{2.6e-08}$ & $\text{5.33e-12}_\text{4.1e-08}$ \\
8 & 4 & 32 & -- & $\text{6.25e-09}_\text{2.5e-09}$ & $\text{2.40e-13}_\text{1.1e-12}$ \vspace{0.5em} \\
8 & 8 & 8 & $\text{7.47e-03}_\text{1.2e-02}$ & $\text{2.45e-02}_\text{9.9e-03}$ & $\text{3.67e00}_\text{2.7e00}$ \\
8 & 8 & 12 & $\text{7.30e-12}_\text{1.2e-10}$ & $\text{1.69e-04}_\text{3.1e-04}$ & -- \\
8 & 8 & 16 & $\text{5.83e-13}_\text{7.0e-12}$ & $\text{2.07e-06}_\text{9.5e-06}$ & $\text{1.77e00}_\text{8.9e-01}$ \\
8 & 8 & 32 & -- & $\text{8.01e-08}_\text{1.1e-07}$ & $\text{3.09e-12}_\text{6.5e-11}$ \\
8 & 8 & 64 & -- & $\text{2.43e-08}_\text{6.6e-09}$ & $\text{3.15e-13}_\text{4.1e-13}$ \vspace{0.5em} \\
8 & 16 & 16 & -- & $\text{2.72e-02}_\text{3.3e-02}$ & $\text{7.07e00}_\text{2.7e00}$ \\
8 & 16 & 24 & -- & $\text{1.61e-03}_\text{1.8e-03}$ & -- \\
8 & 16 & 32 & -- & $\text{8.57e-05}_\text{2.1e-04}$ & $\text{1.47e00}_\text{6.3e-01}$ \\
8 & 16 & 64 & -- & $\text{5.62e-07}_\text{4.8e-07}$ & $\text{2.35e-02}_\text{3.9e-02}$ \\
8 & 16 & 128 & -- & $\text{1.03e-07}_\text{7.6e-08}$ & $\text{3.16e-13}_\text{4.7e-12}$ \\
\midrule
16 & 8 & 8 & -- & $\text{2.79e-02}_\text{4.5e-02}$ & $\text{1.85e01}_\text{6.4e00}$ \\
16 & 8 & 12 & -- & $\text{9.05e-07}_\text{2.2e-05}$ & -- \\
16 & 8 & 16 & -- & $\text{1.10e-07}_\text{2.0e-08}$ & $\text{8.76e00}_\text{2.7e00}$ \\
16 & 8 & 32 & -- & $\text{2.86e-08}_\text{6.4e-09}$ & $\text{1.74e00}_\text{8.2e-01}$ \\
16 & 8 & 64 & -- & $\text{9.53e-09}_\text{2.2e-09}$ & $\text{2.82e-13}_\text{7.9e-11}$ \vspace{0.5em} \\
16 & 16 & 16 & -- & $\text{4.99e-02}_\text{1.6e-02}$ & $\text{2.81e01}_\text{6.4e00}$ \\
16 & 16 & 24 & -- & $\text{4.70e-05}_\text{1.1e-04}$ & -- \\
16 & 16 & 32 & -- & $\text{4.80e-07}_\text{2.1e-06}$ & $\text{1.63e01}_\text{3.0e00}$ \\
16 & 16 & 64 & -- & $\text{8.13e-08}_\text{1.9e-08}$ & $\text{4.03e00}_\text{1.0e00}$ \\
16 & 16 & 128 & -- & $\text{2.75e-08}_\text{3.0e-09}$ & $\text{3.42e-13}_\text{3.8e-12}$ \\
\bottomrule
\end{tabular}
\end{sc}
\end{small}
\end{center}
\vskip -0.1in
\end{table*}

\subsection{Toy Versions of Real-World Datasets}\label{sec:real_world_data_exps}

\subsubsection{Dataset Creation}
Our datasets were created from the MNIST \citep{lecun2010mnist} and Fashion MNIST \citep{xiao2017fashion} datasets.
Both datasets are 10-way multiclass classification datasets; however, our mGLS algorithm only works for ReLU networks with scalar output.
Hence we have to create binary classification tasks from these datasets.

We did this by restricting each dataset to two classes and having the task to correctly differentiate between only those two classes.
For MNIST, we chose the 4 and the 9 classes.
For Fashion MNIST, we chose the pullover and the coat classes.
These classes were chosen for the interclass similarity of their examples, which increases the difficulty of the task.

To reduce the dimensionality of the data, we performed principle components analysis (PCA) using the \texttt{scikit-learn} Python package \citep{pedregosa2011scikit} on all of the training examples in each dataset belonging to their respective two chosen classes.
When then used the first $d \in \{8, 16\}$ whitened components for our dataset.
We then took the first $N \in \{350, 700\}$ examples in the training split as our training dataset.
We always chose the same examples across experiments to reduce variance.

\subsubsection{Training Details}
Since these datasets were binary classification tasks, we used the sigmoid cross entropy loss function $\ell(\hat y, y) = - y\hat y + \sigma(\hat y)$, where $\sigma$ is the logistic sigmoid function.
The network architecture for these experiments takes the form of
\begin{equation}
    f(\x) = \v^T\phi(W\bar\x) + c,
\end{equation}
where $\v\in\R^m$, $W\in\R^{m\times (d+1)}$, and $c\in\R$.
We use $m\in\N$ to denote the number of units in the network that we train on the dataset.
In all experiments in this section, set $\v$ to a vector containing half ones and half negative ones and froze it throughout training.
The rest of the variables were optimized during training.
Note that this is different than what we did for the synthetic datasets.

\paragraph{Gradient Descent}
We trained for one million steps with a learning rate of 1e-3 using batch gradient descent.
We used the parameter initialization scheme from \citet{glorot2010understanding}.

\paragraph{Random Vertex}
A random vertex was selected by sampling the first layer weights $W\in\R^{m\times (d+1)}$ from a standard Gaussian and taking its corresponding activation pattern over the dataset.
We then solved its corresponding convex program \eqref{eq:activation_region_constrained_opt} using the ECOS \citep{domahidi2013ecos} solver in the \texttt{cvxpy} Python package \citep{diamond2016cvxpy}.
We also optimized the bias $c\in\R$ in the final layer as well as the first layer parameters parameters in the convex program.

\paragraph{mGLS Heuristic}
We chose the initial vertex by sampling the first layer weights $W\in\R^{m\times (d+1)}$ from a standard Gaussian and taking its corresponding activation pattern over the dataset.
Like for the random vertex experiment, we also optimized the second layer bias $c\in\R$ in the convex program.
We used the mGLS algorithm presented in \cref{sec:mgls} to perform the optimization.
We set $T=2048$ as the maximum number of steps.

\subsubsection{Full Results}
We experimented with a range of $d,N,m$ values on both MNIST 5/9 and Fashion MNIST coat/pullover.
We present our full results comparing gradient descent to the random vertex method in \cref{table:rw_gd_rv_exp_results} and comparing gradient descent to mGLS in \cref{table:rw_gd_gls_exp_results}.
Random vertex results are the median of 16 runs while the results for the other two methods are the median of 8 runs.

\begin{table*}[t]
\caption{
Results of all experiments comparing gradient descent to random vertex optimization in this paper. The subscripts provide standard deviation across runs.}
\label{table:rw_gd_rv_exp_results}
\vskip 0.15in
\begin{center}
\begin{small}
\begin{sc}
\begin{tabular}{cccc@{\qquad}cc@{\qquad}cc}
\toprule
\multirow{2}{*}{\raisebox{-\heavyrulewidth}{Dataset}} & & & & \multicolumn{2}{c}{Gradient~Descent} & \multicolumn{2}{c}{Random~Vertex} \\
\cmidrule{5-8}
& $d$ & $N$ & $m$ & Loss & Acc~(\%) & Loss & Acc~(\%) \\
\midrule
\multirow{4}{*}{\raisebox{-\heavyrulewidth}{MNIST}}
& 8 & 350 & 4 & $\text{1.16e-01}_\text{1.1e-02}$ & $\text{95.6}_\text{0.80}$ & $\text{6.04e-01}_\text{6.7e-02}$ & $\text{67.4}_\text{8.53}$ \\
 & 8 & 350 & 8 & $\text{5.37e-02}_\text{7.9e-03}$ & $\text{98.9}_\text{0.45}$ & $\text{5.27e-01}_\text{1.1e-01}$ & $\text{72.9}_\text{8.30}$ \\
 & 8 & 350 & 16 & $\text{2.09e-02}_\text{2.6e-03}$ & $\text{99.5}_\text{0.16}$ & $\text{3.21e-01}_\text{7.3e-02}$ & $\text{85.3}_\text{4.87}$ \\
 & 8 & 350 & 32 & $\text{8.38e-03}_\text{5.1e-04}$ & $\text{100.0}_\text{0.00}$ & $\text{2.30e-01}_\text{5.0e-02}$ & $\text{91.4}_\text{2.36}$ \\
 & 8 & 350 & 64 & $\text{4.21e-03}_\text{1.5e-04}$ & $\text{100.0}_\text{0.00}$ & $\text{1.71e-01}_\text{4.9e-02}$ & $\text{94.1}_\text{1.86}$ \vspace{0.5em} \\
 & 8 & 700 & 4 & $\text{1.62e-01}_\text{1.9e-03}$ & $\text{93.2}_\text{0.73}$ & $\text{6.50e-01}_\text{8.7e-02}$ & $\text{63.4}_\text{9.13}$ \\
 & 8 & 700 & 8 & $\text{1.24e-01}_\text{6.7e-03}$ & $\text{95.0}_\text{0.56}$ & $\text{4.94e-01}_\text{8.6e-02}$ & $\text{75.9}_\text{7.11}$ \\
 & 8 & 700 & 16 & $\text{7.90e-02}_\text{6.1e-03}$ & $\text{97.6}_\text{0.49}$ & $\text{3.95e-01}_\text{6.4e-02}$ & $\text{82.1}_\text{4.34}$ \\
 & 8 & 700 & 32 & $\text{3.97e-02}_\text{1.5e-03}$ & $\text{99.4}_\text{0.05}$ & $\text{2.79e-01}_\text{5.7e-02}$ & $\text{88.6}_\text{3.02}$ \\
 & 8 & 700 & 64 & $\text{1.94e-02}_\text{6.5e-04}$ & $\text{100.0}_\text{0.04}$ & $\text{2.14e-01}_\text{2.2e-02}$ & $\text{91.5}_\text{1.06}$ \\
 \cmidrule{3-8}
 & 16 & 350 & 4 & $\text{1.78e-02}_\text{2.2e-03}$ & $\text{99.5}_\text{0.18}$ & $\text{5.95e-01}_\text{7.8e-02}$ & $\text{68.7}_\text{7.10}$ \\
 & 16 & 350 & 8 & $\text{5.75e-03}_\text{7.4e-04}$ & $\text{100.0}_\text{0.00}$ & $\text{4.89e-01}_\text{1.2e-01}$ & $\text{75.4}_\text{9.28}$ \\
 & 16 & 350 & 16 & $\text{2.49e-03}_\text{2.8e-04}$ & $\text{100.0}_\text{0.00}$ & $\text{4.14e-01}_\text{1.1e-01}$ & $\text{80.0}_\text{7.91}$ \\
 & 16 & 350 & 32 & $\text{1.12e-03}_\text{5.0e-05}$ & $\text{100.0}_\text{0.00}$ & $\text{1.39e-01}_\text{9.1e-02}$ & $\text{95.0}_\text{3.95}$ \\
 & 16 & 350 & 64 & $\text{4.87e-04}_\text{2.2e-05}$ & $\text{100.0}_\text{0.00}$ & $\text{4.81e-02}_\text{3.9e-02}$ & $\text{99.1}_\text{1.29}$ \vspace{0.5em} \\
 & 16 & 700 & 4 & $\text{3.51e-02}_\text{4.4e-03}$ & $\text{99.3}_\text{0.18}$ & $\text{6.61e-01}_\text{5.6e-02}$ & $\text{57.7}_\text{7.61}$ \\
 & 16 & 700 & 8 & $\text{1.16e-02}_\text{1.3e-03}$ & $\text{99.9}_\text{0.16}$ & $\text{5.34e-01}_\text{9.0e-02}$ & $\text{72.8}_\text{7.68}$ \\
 & 16 & 700 & 16 & $\text{4.61e-03}_\text{2.9e-04}$ & $\text{100.0}_\text{0.00}$ & $\text{4.37e-01}_\text{8.7e-02}$ & $\text{78.9}_\text{6.01}$ \\
 & 16 & 700 & 32 & $\text{2.06e-03}_\text{1.1e-04}$ & $\text{100.0}_\text{0.00}$ & $\text{2.36e-01}_\text{6.3e-02}$ & $\text{90.1}_\text{3.11}$ \\
 & 16 & 700 & 64 & $\text{1.00e-03}_\text{4.0e-05}$ & $\text{100.0}_\text{0.00}$ & $\text{1.20e-01}_\text{3.3e-02}$ & $\text{95.6}_\text{1.44}$ \\
\midrule
\multirow{4}{*}{\raisebox{-\heavyrulewidth}{Fashion MNIST}}
 & 8 & 350 & 4 & $\text{2.98e-01}_\text{1.0e-02}$ & $\text{88.4}_\text{0.67}$ & $\text{5.97e-01}_\text{7.2e-02}$ & $\text{68.4}_\text{8.60}$ \\
 & 8 & 350 & 8 & $\text{2.24e-01}_\text{1.3e-02}$ & $\text{91.4}_\text{0.71}$ & $\text{5.58e-01}_\text{5.5e-02}$ & $\text{72.9}_\text{6.28}$ \\
 & 8 & 350 & 16 & $\text{1.39e-01}_\text{9.8e-03}$ & $\text{95.9}_\text{0.90}$ & $\text{4.96e-01}_\text{5.7e-02}$ & $\text{78.6}_\text{3.39}$ \\
 & 8 & 350 & 32 & $\text{6.71e-02}_\text{7.4e-03}$ & $\text{98.8}_\text{0.43}$ & $\text{3.78e-01}_\text{2.8e-02}$ & $\text{85.1}_\text{1.68}$ \\
 & 8 & 350 & 64 & $\text{3.21e-02}_\text{2.3e-03}$ & $\text{100.0}_\text{0.19}$ & $\text{2.91e-01}_\text{3.1e-02}$ & $\text{88.1}_\text{1.34}$ \vspace{0.5em}  \\
 & 8 & 700 & 4 & $\text{3.52e-01}_\text{3.8e-03}$ & $\text{84.8}_\text{0.53}$ & $\text{6.62e-01}_\text{5.1e-02}$ & $\text{60.5}_\text{7.12}$ \\
 & 8 & 700 & 8 & $\text{3.15e-01}_\text{6.4e-03}$ & $\text{86.6}_\text{0.63}$ & $\text{5.99e-01}_\text{4.4e-02}$ & $\text{69.3}_\text{4.66}$ \\
 & 8 & 700 & 16 & $\text{2.48e-01}_\text{1.0e-02}$ & $\text{90.5}_\text{0.76}$ & $\text{5.17e-01}_\text{4.2e-02}$ & $\text{76.1}_\text{3.20}$ \\
 & 8 & 700 & 32 & $\text{1.65e-01}_\text{5.0e-03}$ & $\text{94.3}_\text{0.50}$ & $\text{4.31e-01}_\text{3.0e-02}$ & $\text{80.7}_\text{1.78}$ \\
 & 8 & 700 & 64 & $\text{8.70e-02}_\text{6.0e-03}$ & $\text{98.3}_\text{0.37}$ & $\text{3.66e-01}_\text{1.7e-02}$ & $\text{84.4}_\text{1.27}$ \\
 \cmidrule{3-8}
 & 16 & 350 & 4 & $\text{1.93e-01}_\text{1.6e-02}$ & $\text{92.4}_\text{1.25}$ & $\text{6.24e-01}_\text{6.2e-02}$ & $\text{64.4}_\text{9.07}$ \\
 & 16 & 350 & 8 & $\text{8.33e-02}_\text{1.3e-02}$ & $\text{98.0}_\text{0.70}$ & $\text{5.46e-01}_\text{6.0e-02}$ & $\text{72.7}_\text{5.67}$ \\
 & 16 & 350 & 16 & $\text{3.71e-02}_\text{3.7e-03}$ & $\text{99.3}_\text{0.14}$ & $\text{4.27e-01}_\text{5.2e-02}$ & $\text{80.1}_\text{3.29}$ \\
 & 16 & 350 & 32 & $\text{1.43e-02}_\text{3.3e-03}$ & $\text{100.0}_\text{0.20}$ & $\text{3.00e-01}_\text{4.4e-02}$ & $\text{87.3}_\text{2.31}$ \\
 & 16 & 350 & 64 & $\text{6.72e-03}_\text{4.3e-04}$ & $\text{100.0}_\text{0.00}$ & $\text{1.65e-01}_\text{1.0e-01}$ & $\text{94.3}_\text{4.22}$ \vspace{0.5em} \\
 & 16 & 700 & 4 & $\text{2.84e-01}_\text{1.1e-02}$ & $\text{88.1}_\text{0.48}$ & $\text{6.30e-01}_\text{4.6e-02}$ & $\text{64.8}_\text{5.66}$ \\
 & 16 & 700 & 8 & $\text{1.95e-01}_\text{8.5e-03}$ & $\text{92.0}_\text{0.31}$ & $\text{5.95e-01}_\text{4.9e-02}$ & $\text{68.4}_\text{4.37}$ \\
 & 16 & 700 & 16 & $\text{1.13e-01}_\text{7.5e-03}$ & $\text{96.5}_\text{0.58}$ & $\text{5.19e-01}_\text{4.6e-02}$ & $\text{75.7}_\text{3.27}$ \\
 & 16 & 700 & 32 & $\text{4.53e-02}_\text{3.4e-03}$ & $\text{99.4}_\text{0.25}$ & $\text{4.26e-01}_\text{2.4e-02}$ & $\text{81.4}_\text{1.45}$ \\
 & 16 & 700 & 64 & $\text{2.05e-02}_\text{1.5e-03}$ & $\text{100.0}_\text{0.06}$ & $\text{3.43e-01}_\text{1.5e-02}$ & $\text{85.6}_\text{1.14}$ \\
\bottomrule
\end{tabular}
\end{sc}
\end{small}
\end{center}
\vskip -0.1in
\end{table*}

\begin{table*}[t]

\caption{
Results of all experiments comparing gradient descent to mGLS in this paper. The subscripts provide standard deviation across runs.}
\label{table:rw_gd_gls_exp_results}
\vskip 0.15in
\begin{center}
\begin{small}
\begin{sc}
\begin{tabular}{cc@{\qquad}cc@{\qquad}cc}
\toprule
\multirow{2}{*}{\raisebox{-\heavyrulewidth}{Dataset}} &
\multirow{2}{*}{\raisebox{-\heavyrulewidth}{$m$}} & \multicolumn{2}{c}{Gradient~Descent} & \multicolumn{2}{c}{mGLS~Heuristic} \\
\cmidrule{3-6}
& & Loss & Acc~(\%) & Loss & Acc~(\%) \\
\midrule
\multirow{4}{*}{\raisebox{-\heavyrulewidth}{MNIST}}
& 4 & $\text{1.16e-01}_\text{1.1e-02}$ & $\text{95.6}_\text{0.80}$ & $\text{1.09e-01}_\text{3.7e-02}$ & $\text{95.6}_\text{1.76}$ \\
& 8 & $\text{5.37e-02}_\text{7.9e-03}$ & $\text{98.9}_\text{0.45}$ & $\text{1.95e-03}_\text{1.5e-02}$ & $\text{100.0}_\text{0.58}$ \\
& 16 & $\text{2.09e-02}_\text{2.6e-03}$ & $\text{99.5}_\text{0.16}$ & $\text{8.24e-04}_\text{1.2e-03}$ & $\text{100.0}_\text{0.00}$ \\
& 32 & $\text{8.38e-03}_\text{5.1e-04}$ & $\text{100.0}_\text{0.00}$ & $\text{3.73e-02}_\text{8.5e-03}$ & $\text{99.0}_\text{0.61}$ \\
\midrule
\multirow{4}{*}{\raisebox{-\heavyrulewidth}{Fashion MNIST}}
& 4 & $\text{2.98e-01}_\text{1.0e-02}$ & $\text{88.4}_\text{0.67}$ & $\text{2.88e-01}_\text{1.8e-02}$ & $\text{88.3}_\text{0.64}$ \\
& 8 & $\text{2.24e-01}_\text{1.3e-02}$ & $\text{91.4}_\text{0.71}$ & $\text{1.72e-01}_\text{4.5e-02}$ & $\text{93.7}_\text{1.99}$ \\
& 16 & $\text{1.39e-01}_\text{9.8e-03}$ & $\text{95.9}_\text{0.90}$ & $\text{2.91e-03}_\text{1.5e-02}$ & $\text{100.0}_\text{0.47}$ \\
& 32 & $\text{6.71e-02}_\text{7.4e-03}$ & $\text{98.8}_\text{0.43}$ & $\text{2.33e-02}_\text{6.1e-02}$ & $\text{99.4}_\text{2.45}$ \\
\bottomrule
\end{tabular}
\end{sc}
\end{small}
\end{center}
\vskip -0.1in
\end{table*}

\end{document}